\documentclass[runningheads]{llncs}

\usepackage{algorithm}
\usepackage{algorithmic}
\usepackage{xspace}
\usepackage{amsmath,amssymb,bm}
\usepackage{enumitem}
\usepackage{thmtools}
\usepackage{siunitx}
\usepackage{wrapfig}
\usepackage{placeins}
\usepackage{dirtytalk}

\usepackage{mathtools}

\def\1{\bm{1}}

\newcommand{\test}{\mathcal{D_{\mathrm{test}}}}

\def\va{{\bm{a}}}

\def\vr{{\bm{r}}}

\def\vt{{\bm{t}}}

\def\vx{{\bm{x}}}

\def\vz{{\bm{z}}}

\DeclareMathAlphabet{\mathsfit}{\encodingdefault}{\sfdefault}{m}{sl}
\SetMathAlphabet{\mathsfit}{bold}{\encodingdefault}{\sfdefault}{bx}{n}

\DeclareMathOperator*{\argmax}{arg\,max}

\DeclareMathOperator{\sign}{sign}

\newcommand{\prob}[0]{\mathbb{P}} %

\usepackage{comment}
\usepackage{color}

\DeclareMathOperator{\bce}{cross\_entropy}

\usepackage[utf8]{inputenc} %
\usepackage{lipsum} %

\usepackage[T1]{fontenc}

\usepackage{array}
\newcolumntype{P}[1]{>{\centering\arraybackslash}p{#1}}

\usepackage{url}            %
\usepackage{booktabs}       %
\usepackage{amsfonts}       %
\usepackage{nicefrac}       %
\usepackage{microtype}      %
\usepackage{multirow}
\usepackage{graphicx}
\usepackage{caption}
\usepackage[caption=false]{subfig}

\usepackage{etoolbox}
\newbool{includeappendix}
\setbool{includeappendix}{true}

\ifdefined\isoverfull
\else
\fi

\newcommand{\eg}{e.g., }
\newcommand{\ie}{i.e., }
\newcommand{\etc}{etc.}

\newcommand{\wrt}{{w.r.t.\ }}

\newcommand{\producer}{P}
\newcommand{\consumer}{C}
\newcommand{\lassi}{R}
\newcommand{\model}{M}
\newcommand{\glowenc}{E}
\newcommand{\glowdec}{D}
\newcommand{\smoothconsumer}{\widehat{\consumer}}
\newcommand{\smoothc}{\widehat{\consumer}}
\newcommand{\smoothlassi}{\widehat{\lassi}}

\newcommand{\inputspace}{\mathbb{R}^n}
\newcommand{\representationspace}{\mathbb{R}^k}
\newcommand{\generativespace}{\mathbb{R}^q}
\newcommand{\outputspace}{\mathcal{Y}}
\newcommand{\inputsim}{\phi}

\newcommand{\inputpoint}{\vx}
\newcommand{\representationpoint}{\vr}
\newcommand{\attrvec}{\va}
\newcommand{\generativepoint}{\vz}
\newcommand{\glowpoint}{\generativepoint_G}
\newcommand{\lcifrpoint}{\representationpoint_R}
\newcommand{\CScenter}{\representationpoint_{cs}}
\newcommand{\CSradius}{d_{cs}}
\newcommand{\CSsigma}{\sigma_{cs}}
\newcommand{\RSradius}{d_{rs}}
\newcommand{\RSsigma}{\sigma_{rs}}
\newcommand{\CSprob}{\alpha_{cs}}
\newcommand{\RSprob}{\alpha_{rs}}
\newcommand{\simset}{S\left(\inputpoint\right)}
\newcommand{\simsetl}{\simset}
\newcommand{\simseti}{S^{\mathrm{in}}\left(\inputpoint\right)}

\usepackage{tikz}

\usetikzlibrary{arrows}
\usetikzlibrary{automata}
\usetikzlibrary{calc}
\usetikzlibrary{backgrounds}
\usetikzlibrary{decorations.markings}
\usetikzlibrary{decorations.pathmorphing}
\usetikzlibrary{decorations.pathreplacing}
\usetikzlibrary{fit}
\usetikzlibrary{patterns}
\usetikzlibrary{positioning}
\usetikzlibrary{shadows}
\usetikzlibrary{shapes}
\usetikzlibrary{shapes.geometric}

\usepackage{orcidlink}

\usepackage{hyperref}

\hypersetup{
  colorlinks   = true, %
  urlcolor     = blue, %
  linkcolor={red!50!black},
  linkcolor={red!50!black}, %
  citecolor={blue!50!black} %
}

\usepackage[capitalize]{cleveref}

\crefrangeformat{section}{\S#3#1#4\crefrangeconjunction\S#5#2#6}

\crefmultiformat{section}{\S#2#1#3}{\crefpairconjunction\S#2#1#3}{\crefmiddleconjunction\S#2#1#3}{\creflastconjunction\S#2#1#3}

\newcommand{\crefrangeconjunction}{--}

\crefname{listing}{Lst.}{listings}
\crefname{algorithm}{Alg.}{algorithms}
\crefname{theorem}{Thm.}{theorems}
\Crefname{theorem}{Theorem}{theorems}
\crefname{line}{Lin.}{Lin.}
\crefname{appendix}{App.}{App.}

\newcommand{\app}[1]{%
	\ifbool{includeappendix}{\cref{#1}}{the appendix}%
}
\newcommand{\App}[1]{%
	\ifbool{includeappendix}{\cref{#1}}{The appendix}%
}

\crefname{section}{Sec.}{Secs.}
\Crefname{section}{Section}{Sections}
\Crefname{table}{Table}{Tables}
\crefname{table}{Tab.}{Tabs.}

\begin{document}
\pagestyle{headings}
\mainmatter

\title{Latent Space Smoothing for\\Individually Fair Representations} %

\titlerunning{Latent Space Smoothing for Individually Fair Representations}
\author{Momchil Peychev$^1$\orcidlink{0000-0003-0927-6356} \and
Anian Ruoss$^2$\thanks{Work partially done while the author was at ETH Zurich.}\orcidlink{0000-0002-8616-2558} \and
Mislav Balunovi\'{c}$^1$\orcidlink{0000-0002-7024-7688} \and\\
Maximilian Baader$^1$\orcidlink{0000-0002-9271-6422} \and
Martin Vechev$^1$\orcidlink{0000-0002-0054-9568}
}
\authorrunning{M. Peychev et al.}
\institute{$^1$Department of Computer Science, ETH Zurich \quad $^2$DeepMind, London\\
\email{\{momchil.peychev,mislav.balunovic,mbaader,martin.vechev\}@inf.ethz.ch
anianr@deepmind.com}
}
\maketitle

\begin{abstract}
  Fair representation learning transforms user data into a representation that
ensures fairness and utility regardless of the downstream application.
However, learning individually fair representations, \ie guaranteeing that
similar individuals are treated similarly, remains challenging in
high-dimensional settings such as computer vision.
In this work, we introduce LASSI, the first representation learning method for
certifying individual fairness of high-dimensional data.
Our key insight is to leverage recent advances in generative modeling to capture
the set of similar individuals in the generative latent space.
This enables us to learn individually fair representations that map similar
individuals close together by using adversarial training to minimize the
distance between their representations.
Finally, we employ randomized smoothing to provably map similar individuals
close together, in turn ensuring that local robustness verification of the
downstream application results in end-to-end fairness certification.
Our experimental evaluation on challenging real-world image data demonstrates
that our method increases certified individual fairness by up to 90\%
without significantly affecting task utility.

\end{abstract}

\section{Introduction}
\label{sec:introduction}

Deep learning models are increasingly deployed in critical domains, such as face
detection~\cite{sun2018face}, credit scoring~\cite{khandani2010consumer}, or
crime risk assessment~\cite{brennan2009evaluating}, where decisions of the model
can have wide-ranging impacts on society.
Unfortunately, the models and datasets employed in these settings are
biased~\cite{buolamwini2018gender,klare2012face}, which raises concerns against
their usage for such tasks and causes regulators to hold organizations
accountable for the discriminatory effects of their
models~\cite{eu2019ethics,eu2021proposal,ftc2020using,ftc2021aiming,hrc2021right}.

In this regard, fair representation learning~\cite{zemel2013learning} is a
promising bias mitigation approach that transforms data to prevent
discrimination regardless of the concrete downstream application while
simultaneously maintaining high task utility.
The approach is highly modular~\cite{mcnamara2019costs}: the \emph{data
regulator} defines the fairness notion, the \emph{data producer} learns a fair
representation that encodes the data, and the \emph{data consumers} employ the
transformed data in downstream tasks.
Recent work successfully augmented fair representation learning with
guarantees~\cite{gitiaux2021learning,ruoss2020learning}, but its application to
high-dimensional data, such as images, remains challenging.

\medskip

\noindent\textbf{Key challenge: scaling to high-dimensional data and
real-world models} {}
The two central challenges of \emph{individually} fair representation learning,
which requires similar individuals to be treated similarly, are: (i) designing
a suitable input similarity
metric~\cite{yurochkin2020training,zemel2013learning} and (ii) enforcing that
similar individuals are \emph{provably} treated similarly according to that
metric.
For low-dimensional tabular data, prior work has typically measured input
similarity in terms of the input features (age, income, \etc), using, \eg logical
constraints~\cite{ruoss2020learning} or weighted
$\ell_p$-metrics~\cite{yeom2020individual}.
However, characterizing the similarity of high-dimensional data, such as images,
at the input-level, \eg by comparing pixels, is infeasible.
Moreover, proving that all points in the infinite set of similar individuals
obtain the same classification requires propagating this set through the model.
Unfortunately, for high-dimensional applications this is unattainable for prior
work using (mixed-integer) linear programming
solvers~\cite{ehlers2017formal,tjeng2019evaluating}, which only scale to small
networks.

\medskip

\begin{figure}[t]
	\centering
	\resizebox{\linewidth}{!}{ \input{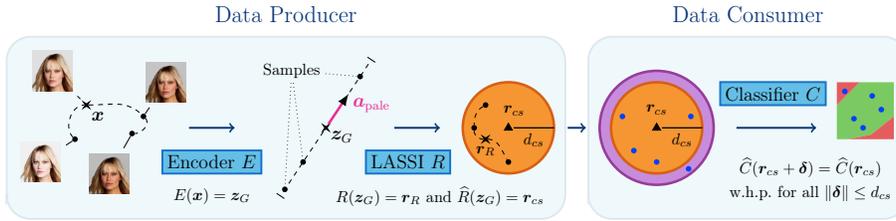} }
	\caption{
		Overview of our framework LASSI.
		The left part shows the data producer who captures the set of individuals
		similar to $\vx$ by interpolating along the attribute vector
		$\attrvec_{\text{pale}}$.
		The data producer then uses adversarial training
		and center smoothing to compute a
		representation that provably maps all similar points into the
		$\ell_2$-ball of radius $\CSradius$ around $\CScenter$.
		The right part shows the data consumer who can certify individual
		fairness, \ie prove that all similar individuals receive the same
		classification outcome, of the end-to-end model by checking whether the
		certified radius obtained via randomized smoothing exceeds $\CSradius$.
	}
	\label{fig:overview}
\end{figure}

\noindent\textbf{This work} {}
In this work, we introduce latent space smoothing for individually fair
representations (LASSI), a method that addresses both of the above challenges.
Our approach leverages two recent advances: the emergence of powerful generative
models~\cite{kingma2018glow}, which enable the definition of image similarity
for individual fairness, and the scalable certification of deep
models~\cite{cohen2019certified}, which allows proving individual fairness.
A high-level overview of our approach is shown in \cref{fig:overview}.
Concretely, we use generative modeling~\cite{kingma2018glow} to enable data
regulators to define input similarity by varying a continuous attribute of the
image, such as pale skin in \cref{fig:overview}.
To enforce that similar individuals are provably treated similarly, we further
base our approach on smoothing: (i) the data producer uses center
smoothing~\cite{kumar2021center} to learn a representation that provably maps
similar individuals close together, and (ii) the data consumer certifies local
$\ell_2$-robustness using randomized smoothing~\cite{cohen2019certified},
thereby proving individual fairness of the end-to-end model.
Therefore, our approach enables data regulators to impose fairness notions of
the form: \emph{\say{For a given person, all people differing only in skin tone
should receive the same classification}} and allows data producers and
consumers to independently learn a representation and classification models that
provably enforce this notion.

To measure input similarity, the data producer leverages the ability of a
bijective generative model to interpolate along the direction of an attribute
vector in the latent space, which is impractical in the pixel space.
As a result, the set of similar individuals can be defined by a line segment in
the latent space (center part of the data producer in \cref{fig:overview}),
corresponding to an elaborate curve in the input space (left part of the data
producer in \cref{fig:overview}), which cannot be concisely captured by, \eg an
$\ell_p$-ball.
Thus, the data producer learns a representation $\lassi$ that maps all points of
the latent line segment close together in the representation space by using
adversarial training to minimize the distance between similar individuals.
However, as adversarial training cannot provide guarantees on this maximum
distance, the data producer uses center smoothing~\cite{kumar2021center} to
adjust the representation such that its \emph{smoothed} version $\smoothlassi$
provably maps all similar points into an $\ell_2$-ball of radius $\CSradius$
around a center $\CScenter$ with high probability (right part of the data
producer in \cref{fig:overview}).
Finally, the data consumer only needs to prove that the certified radius
(violet in the data consumer part of \cref{fig:overview}) of its \emph{smoothed}
classifier $\smoothconsumer$ around $\CScenter$ is larger than $\CSradius$ to
obtain an individual fairness certificate for the end-to-end model
$\model := \smoothconsumer \circ \smoothlassi \circ \glowenc$.

Our experimental evaluation on several image classification tasks shows that
training with LASSI significantly increases the number of individuals for which
we can certify individual fairness, with respect to multiple different
sensitive attributes, as well as their combinations. Overall, we certify up to
90\% more than the baselines.
Furthermore, we demonstrate that the representations obtained by LASSI can be
used to solve classification tasks that were unseen during training.

\medskip

\noindent\textbf{Main contributions} {} {} We make the following contributions:
\begin{itemize}[noitemsep,topsep=0pt]
    \item A novel input similarity metric for high-dimensional data defined via
    interpolation in the latent space of generative models.
    \item A scalable representation learning method with individual fairness
    certification for models using high-dimensional data via randomized smoothing.
    \item A large-scale evaluation of our method on various image classification tasks.
\end{itemize}

\section{Related Work}
\label{sec:related-work}

In this work, we consider individual fairness, which requires that similar
individuals be treated similarly~\cite{dwork2012fairness}.
In contrast, group fairness enforces specific classification statistics to be
equal across different groups of the
population~\cite{dwork2012fairness,hardt2016equality}.
While both fairness notions are desirable, they also both suffer from certain
shortcomings.
For instance, models satisfying group fairness may still discriminate against
individuals~\cite{dwork2012fairness} or subgroups~\cite{kearns2018preventing}.
In contrast, the central challenge limiting practical adoption of individual
fairness is the lack of a widely accepted similarity
metric~\cite{yurochkin2020training}.
While recent work has made progress in developing similarity metrics for tabular
data~\cite{ilvento2020metric,maity2021statistical,mukherjee2020two,wang2019empirical,yurochkin2021sensei},
defining similarity concisely for high-dimensional data remains challenging and
is a key contribution of our work.

\medskip

\noindent\textbf{Fair representation learning} {}
A wide range of methods has been proposed to learn fair representations of user
data.
Most of these works consider group fairness and employ techniques such as
adversarial
learning~\cite{edwards2016censoring,kehrenberg2020null,liao2019learning,madras2018learning},
disentanglement~\cite{creager2019flexibly,locatello2019fairness,sarhan2020fairness},
duality~\cite{song2019learning}, low-rank matrix
factorization~\cite{oneto2020learning}, and distribution
alignment~\cite{balunovic2021fair,louizos2016variational,zhao2020conditional}.
Fair representation learning for individual fairness has recently gained
attention, with similarity metrics based on logical
formulas~\cite{ruoss2020learning}, Wasserstein
distance~\cite{feng2019learning,lahoti2019ifair}, fairness
graphs~\cite{lahoti2019operationalizing}, and weighted
$\ell_p$-norms~\cite{zemel2013learning}.
Unfortunately, none of these approaches can capture the similarity between
individuals for the high-dimensional data we consider in our work.

\medskip

\noindent\textbf{Bias in high-dimensional data} {}
A long line of work has investigated the biases of models operating on
high-dimensional data, such as
images~\cite{wang2020towards,wilson2019predictive} and
text~\cite{bolukbasi2016man,liang2021towards,park2018reducing,tatman2017gender},
showing, \eg that black women obtain lower
accuracy in commercial face
classification~\cite{buolamwini2018gender,klare2012face,raji2019actionable}.
Importantly, these models not only learn but also amplify the biases of the
training data~\cite{hendricks2018women,zhao2017men}, even for balanced
datasets~\cite{wang2019balanced}.
A key challenge for bias mitigation in high-dimensional settings is that, unlike
tabular data, sensitive attributes such as age or skin tone are not directly
encoded as features.
Thus, prior work has often relied on generative
models~\cite{balakrishnan2020towards,dash2020counterfactual,denton2019detecting,joo2020gender,kim2018interpretability,kim2021counterfactual,lang2021explaining,li2021discover,ramaswamy2021fair,sattigeri2019fairness}
or computer simulations~\cite{mcduff2018indentifying} to manipulate these
sensitive attributes and check whether the perturbed instances are classified
the same.
However, unlike our work, these methods only tested for bias empirically and
do not provide fairness guarantees.
Recent work also explored using generative models to define
\cite{gowal2020achieving,wong2021learning} or
certify~\cite{mirman2021robustness} robustness, but without focusing on fairness.

\medskip

\noindent\textbf{Fairness certification} {}
Regulatory agencies are increasingly holding organizations accountable for
the discriminatory effects of their machine learning
models~\cite{eu2019ethics,eu2021proposal,ftc2020using,ftc2021aiming,hrc2021right}.
Accordingly, designing algorithms with fairness guarantees has become an active
area of
research~\cite{albarghouthi2017fairsquare,balunovic2021fair,bastani2019probabilistic,choi2021group,gitiaux2021learning,segal2021fairness}.
However, unlike our work, most approaches for individual fairness certification
consider pretrained models and thus cannot be employed in fair representation
learning~\cite{john2020verifying,urban2020perfectly,yeom2020individual}.
In contrast, \cite{ruoss2020learning} learn individually fair representations
with provable guarantees for low-dimensional tabular data, providing a basis
for our approach.
However, neither the similarity notions nor the certification methods employed
by \cite{ruoss2020learning} scale to high-dimensional data, which is the
primary focus of our work.

\section{Background}
\label{sec:background}

This section provides the necessary background on individual fairness, fair
representation learning, generative modeling, and randomized smoothing.

\medskip

\noindent\textbf{Individual fairness} {}
The seminal work of \cite{dwork2012fairness} defined individual fairness as
\say{treating similar individuals similarly}.
In this work, we consider the concrete instantiation of this notion from
\cite{ruoss2020learning}: an individual $\inputpoint'$ is similar to
$\inputpoint$ with respect to a binary input similarity metric
$\inputsim \colon \inputspace \times \inputspace \to \left\{0, 1\right\}$
if and only if $\inputsim(\inputpoint, \inputpoint') = 1$.
A model $\model \colon \inputspace \to \outputspace$ is individually fair
at $\inputpoint \in \inputspace$ if it classifies all individuals similar to
$\inputpoint$ (as measured by $\inputsim$) the same, \ie
\begin{equation}
    \forall \inputpoint' \in \inputspace \colon
    \inputsim\left(\inputpoint, \inputpoint'\right)
    \implies
        \model\left(\inputpoint\right) = \model\left(\inputpoint'\right).
    \label{eq:individual-fairness}
\end{equation}
For example, a credit rating algorithm is individually fair for a given
person if all similar applicants (\eg similar income and repayment history)
receive the same credit rating.
Our goal is to learn a model $\model$ that maximizes the number of points
$\inputpoint$ from the distribution for which we can \emph{guarantee} that
\cref{eq:individual-fairness} is satisfied.
Defining a suitable input similarity metric $\inputsim$ is one of the key
challenges limiting practical applications of individual fairness, and in
\cref{sec:similarity} we will show how to employ generative modeling to
overcome this obstacle for high-dimensional data.

\medskip

\noindent\textbf{Fair representation learning} {}
Fair representation learning~\cite{zemel2013learning} partitions the model
$\model \colon \inputspace \to \outputspace$ into a data producer
$\producer \colon \inputspace \to \representationspace$, which maps input points
$\inputpoint \in \inputspace$ into a representation space $\representationspace$
that satisfies a given fairness notion while maintaining downstream utility, and
a data consumer
$\consumer \colon \representationspace \to \outputspace$ that solves a downstream
task taking only the transformed data points
$\representationpoint \coloneqq \producer\left(\inputpoint\right) \in
\representationspace$ as inputs.
Importantly, the consumers (potentially indifferent to fairness) can
employ standard training methods to obtain fair classifiers that are useful
across a variety of different tasks.
We base our approach on the LCIFR framework~\cite{ruoss2020learning}, which learns
representations with individual fairness guarantees for low-dimensional tabular
data.
LCIFR defines a family of similarity notions and leverages (mixed-integer)
linear programming methods for fairness certification.
However, high-dimensional applications are out of reach for LCIFR because both
the similarity notions and linear programming methods are tailored to
low-dimensional tabular data.
In particular, similarity is defined via logical formulas operating on the
features of $\inputpoint$, which is infeasible for, \eg images, which cannot be
compared solely at the pixel level.
Moreover, while linear programming methods work well for small networks,
they do not scale to real-world computer vision models.
In this work, we show how to resolve these two key concerns to generalize the
high-level idea of LCIFR to real-world, high-dimensional applications.

\medskip

\noindent\textbf{Generative modeling} {}
Normalizing flows, such as Glow~\cite{kingma2018glow}, recently emerged as a
promising generative modeling approach due to their exact log-likelihood
evaluation, efficient inference and synthesis, and useful latent space for
downstream tasks.
Unlike GANs~\cite{goodfellow2014generative} or VAEs~\cite{kingma2014auto},
normalizing flows are bijective models consisting of an encoder
$\glowenc \colon \inputspace \rightarrow \generativespace$ and a decoder
$\glowdec \colon \generativespace \rightarrow \inputspace$ for which
$\inputpoint = \glowdec\left(\glowenc\left(\inputpoint\right)\right)$.
Glow's input space $\mathbb{R}^n$ and latent space $\mathbb{R}^q$ have
the same dimensionalities $n = q$.
Its latent space captures important data attributes, thus enabling latent
space interpolation such as changing the age of a person in an image.
While attribute manipulation via latent space interpolation has also been
investigated in the fairness context for GANs and
VAEs~\cite{balakrishnan2020towards,denton2019detecting,joo2020gender,kim2018interpretability,li2021discover,ramaswamy2021fair},
Glow's key advantages are the existence of an encoder (unlike GANs, which
cannot represent an input point in the latent space efficiently) and the
bijectivity of the end-to-end model (VAEs cannot reconstruct the input point
exactly).
Our key idea is to leverage Glow to define image similarity by interpolating
along the directions defined by certain sensitive attributes in the latent
space.

\medskip

\noindent\textbf{Smoothing} {}
Unlike (mixed-integer) linear
programming~\cite{ehlers2017formal,tjeng2019evaluating}, smoothing
approaches~\cite{cohen2019certified} can compute local robustness guarantees for
any type of classifier
$\consumer \colon \representationspace \to \outputspace$, regardless of its
complexity and scale.
To that end, \cite{cohen2019certified} construct a smoothed classifier
$\smoothconsumer \colon \representationspace \to \outputspace$, which returns the
most probable classification of $\consumer$ for an input
$\representationpoint \in \representationspace$ when perturbed by random noise
from $\mathcal{N}(0, \RSsigma^2 I)$.
Using a sampling-based approach, \cite{cohen2019certified} establish a local
robustness guarantee of the form: $\forall \bm{\delta} \in \representationspace$
such that $\|\bm{\delta}\|_2 < \RSradius$ we have
$\smoothconsumer\left(\representationpoint + \bm{\delta}\right) =
\smoothconsumer\left(\representationpoint\right)$ with probability
$1 - \RSprob$, where $\RSprob$ can be made arbitrarily small.
Thus, $\smoothconsumer$ will classify all points in the $\ell_2$-ball of radius
$\RSradius$ around $\representationpoint$ the same with high probability.
Recently, \cite{kumar2021center} introduced center smoothing, which extends this
approach from classification to multidimensional regression.
Concretely, for a function
$\lassi \colon \generativespace \to \representationspace$, center smoothing uses
sampling and approximation to compute a smooth version
$\smoothlassi \colon \generativespace \to \representationspace$, which maps
$\generativepoint \in \generativespace$ to the center point
$\CScenter \coloneqq \smoothlassi\left(\generativepoint\right)$ of a minimum enclosing
ball containing at least half of the points
$\representationpoint_i \sim \lassi(\generativepoint + \mathcal{N}(0, \CSsigma^2 I))$
for $i \in \{1, \ldots, m\}$.
Then, for $\epsilon > 0$ and $\forall \generativepoint' \in \generativespace$
such that $\|\generativepoint - \generativepoint'\|_2 \leq \epsilon$, we have
$\|\smoothlassi\left(\generativepoint\right) -
\smoothlassi\left(\generativepoint'\right)\|_2 \leq \CSradius$ with probability
at least $1 - \CSprob$.
That is, center smoothing computes a sound upper bound $\CSradius$ on the
$\ell_2$-ball of the function outputs of $\smoothlassi$ for all points in the
$\ell_2$-ball of radius $\epsilon$ around $\generativepoint$.

\section{High-Dimensional Individually Fair Representations}
\label{sec:lassi}

In this section, we describe how our method defines a set of similar
individuals (\cref{sec:similarity}), learns individually fair representations for
these points (\cref{sec:learning}), and finally, certifies individual fairness
for them (\cref{sec:certification}).
Our approach is general, but we focus on images for presentational purposes.

\subsection{Similarity via a Generative Model}
\label{sec:similarity}

We consider two individuals $\inputpoint$ and $\inputpoint'$ to be similar if
they differ only in their continuous sensitive attributes.
However, semantic attributes, such as skin color, cannot be captured
conveniently via the input features of $\inputpoint$.
Thus, our key idea is to define similarity in the latent space of a generative
model $G$.
We compute a vector $\attrvec \in \generativespace$ associated with the
sensitive attribute, such that interpolating along the direction of $\attrvec$
in the latent space and reconstructing back to the input space results in a
meaningful semantic transformation of that attribute.
There is active research investigating different ways of computing
$\attrvec$~\cite{denton2019detecting,higgins2017betaVAE,kingma2018glow,li2021discover,ramaswamy2021fair},
and we will empirically show that our method is compatible with any such method.

\medskip

\noindent\textbf{Computing $\attrvec$} {}
We define individual similarity in the latent space of Glow~\cite{kingma2018glow}.
Our method is independent of the actual computation of $\attrvec$, which we
demonstrate by instantiating four different attribute vector types.
Let $\glowpoint = \glowenc(\inputpoint)$ be the latent code of $\inputpoint$ in
the generative latent space.
First, following \cite{kingma2018glow}, we compute $\attrvec$ by calculating the
average latent vectors $\generativepoint_{G, pos}$ for samples with the attribute and
$\generativepoint_{G, neg}$ for samples without it and set $\attrvec$ to their
difference, $\attrvec = \generativepoint_{G, pos} - \generativepoint_{G, neg}$.
Second, following \cite{denton2019detecting}, we train a linear classifier
$\sign(\attrvec^{\top}\glowpoint + b)$ to predict the presence of the attribute
from $\glowpoint$ and take $\attrvec$ to be the vector orthonormal to the
decision boundary of the linear classifier.
Finally, we employ~\cite{li2021discover} and~\cite{ramaswamy2021fair} who
build on these methods,
accounting for the possible correlations between the sensitive and target attributes.
In all cases, moving in one direction of $\attrvec$ in the latent space
increases the presence of the attribute and interpolating in the opposite
direction decreases it. LASSI is independent of the sensitive attribute
vector computation and will immediately benefit from all advancements in this
area. We evaluate with vectors computed
by~\cite{kingma2018glow} and~\cite{denton2019detecting} in the main paper
(\cref{sec:experimental-evaluation}) and present further results
with vectors from~\cite{li2021discover,ramaswamy2021fair}
in~\cref{app:more-attrib-vectors}.

\medskip

\begin{wrapfigure}{R}{0.21\linewidth}
	\centering
	\resizebox{0.2\textwidth}{!}{ \input{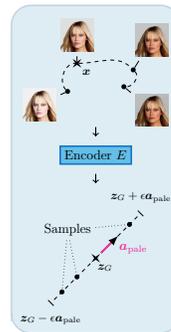} }
	\caption{Similarity in latent space.}
	\label{fig:similarity}
\end{wrapfigure}

\noindent\textbf{Individual similarity in latent space} {}
Using the generative model $G$ and the attribute vector $\attrvec$, we define
the set of individuals similar to $\inputpoint$ in the latent space of $G$ as
$\simset \coloneqq \left\{
	\glowpoint + t\cdot\attrvec \mid |t| \leq \epsilon
\right\} \subseteq \generativespace$
(bottom of \cref{fig:similarity}).
Here, $\epsilon$ denotes the maximum perturbation level
applied to the attribute.
We consider $G$, $\attrvec$, and $\epsilon$ to be a part of the similarity
specification set by the data regulator.
Crucially, $\simset$ contains an infinite number of points
but is compactly represented in the latent space of $G$ as a line segment.
In contrast, the same set represented directly in the input space,
$\simseti \coloneqq \glowdec\left(\simset\right) \subseteq \inputspace$,
obtained by decoding the latent representations in $\simset$ with
$\glowdec$, cannot be abstracted conveniently (top of \cref{fig:similarity}).
Moreover, this approach for constructing $\simset$ can be extended to multiple
sensitive attributes by interpolating along their attribute vectors
simultaneously.
Referring back to the notation in~\cref{sec:background}, we formally define the
input similarity metric $\inputsim$ to satisfy
$\inputsim\left(\inputpoint, \inputpoint'\right) \iff \inputpoint' \in \simseti$.

\subsection{Learning Individually Fair Representations}
\label{sec:learning}

Assuming that the generative model $G = (\glowenc, \glowdec)$ is pretrained
and given (\eg by the data regulator), in this section we describe the learning
of the representation $\lassi \colon \generativespace \to \representationspace$,
which maps from the generative latent space $\generativespace$ directly to the
representation space $\representationspace$.
The representation $\lassi$ is trained separately from the data consumer, the
classifier $\consumer$, whose training is explained in the next section.

\medskip

\noindent\textbf{Adversarial loss} {}
We encourage similar treatment for all points in $\simseti$ by training $\lassi$
to map them close to each other in $\representationspace$,
minimizing the loss
\begin{equation}
	\mathcal{L}_{adv}\left(\inputpoint\right) =
	\max_{\generativepoint' \in \simset}
		\|
			\lassi\left( \glowpoint \right) - \lassi\left( \generativepoint' \right)
		\|_2.
\end{equation}
Minimizing $\mathcal{L}_{adv}\left(\inputpoint\right)$ is a min-max optimization
problem, and adversarial training~\cite{madry2018towards} is known to work well
in such settings.
Because the underlying domain of the inner maximization problem is simply the
line segment $\simset$, we perform a random adversarial attack in which we
sample $s$ points $\generativepoint_i \sim \mathcal{U}\left(\simset\right)$ uniformly at
random from $\simset$ and approximate
$\mathcal{L}_{adv}\left(\inputpoint\right) \approx \max_{i=1}^s \|
\lassi\left( \glowpoint \right) - \lassi\left( \generativepoint_{i} \right)
\|_2$.
This efficient attack is typically more effective~\cite{engstrom2019exploring}
than the first-order methods such as FGSM~\cite{goodfellow2015explaining} and
PGD~\cite{madry2018towards} when the search space is low-dimensional.

\medskip

\noindent\textbf{Classification loss} {}
To ensure that the learned representations remain useful for downstream tasks,
we introduce an auxiliary classifier $C_{aux}$ to predict a ground truth target
label $y$ by adding an additional classification loss term:
\begin{equation}
	\mathcal{L}_{cls}\left(\inputpoint, y\right) = \bce \big(
		C_{aux} \circ \lassi\left(\glowpoint\right), y
		\big).
\end{equation}

\medbreak

\noindent\textbf{Reconstruction loss} {}
The downstream task may not always be known to the data producer a priori, and
thus our representations should ideally transfer to a variety of such tasks.
To that end, we optionally utilize a reconstruction loss, which is designed to
preserve the signal from the original
data~\cite{madras2018learning,ruoss2020learning}:
\begin{equation}
	\mathcal{L}_{recon}\left(\inputpoint\right) =
	\|
		\glowpoint - Q\left(\lassi\left(\glowpoint\right)\right)
	\|_2,
\end{equation}
where $Q \colon \representationspace \to \generativespace$ denotes a reconstruction
network.

The representation $\lassi$, the auxiliary classifier $\consumer_{aux}$, and the
reconstruction network $Q$ are trained jointly using stochastic gradient descent
to minimize the combined objective
\begin{equation}
	\lambda_1\mathcal{L}_{cls}\left(\inputpoint, y\right) +
	\lambda_2\mathcal{L}_{adv}\left(\inputpoint\right) +
	\lambda_3\mathcal{L}_{recon}\left(\inputpoint\right).
\end{equation}
Trading off fairness, accuracy, and transferability is a multi-objective
optimization problem, an active area of research.
Here, we follow~\cite{madras2018learning,ruoss2020learning} and use a
linear scalarization scheme, with the hyperparameters $\lambda_1$, $\lambda_2$
and $\lambda_3$ balancing the three losses, but our method is also compatible
with other schemes~\cite{lin2019pareto,martinez2020minimax,wei2020fairness}.

\subsection{Training Classifier $\consumer$}

Once we have learned the representation $\lassi$, we can use it to train any
classifier $\consumer$ (often different from the auxiliary one $\consumer_{aux}$).
As we will apply smoothing to $\consumer$, we train it by adding
isotropic Gaussian noise to its inputs during the training process, as
in~\cite{cohen2019certified}.
We use the outputs of $\lassi \circ \glowenc$ (and not the smoothed version
$\smoothlassi \circ \glowenc$) as inputs to train $\consumer$, since repeatedly
smoothing the pipeline at this step is computationally expensive and because the
distance between the smoothed and the unsmoothed outputs is generally
small~\cite{kumar2021center}.

\subsection{Certifying Individual Fairness via Latent Space Smoothing}
\label{sec:certification}

With $\lassi$ and $\consumer$ trained as described above, we now construct the
end-to-end model $\model \colon \inputspace \to \outputspace$ for which, given an
input $\inputpoint$, we can certify individual fairness of the form
\begin{equation}
	\forall \inputpoint' \in \simseti:
    \model\left(\inputpoint\right) = \model\left(\inputpoint'\right),
	\label{eq:ind-fairness-certificate-input-space}
\end{equation}
with arbitrarily high probability.

\begin{wrapfigure}{R}{0.2\textwidth}
	\centering
    \resizebox{0.2\textwidth}{!}{ \begin{tikzpicture}

    \definecolor{netinside}{RGB}{100, 190, 230}
    \definecolor{netborder}{RGB}{20, 115, 200}

    \definecolor{Pinside}{RGB}{245, 150, 50}
    \definecolor{Pborder}{RGB}{215, 95, 20}

    \definecolor{DPinside}{RGB}{98, 189, 222}
    \definecolor{DPinsideSmall}{RGB}{56, 151, 186}
    \definecolor{DPborder}{RGB}{6, 51, 117}

    \definecolor{CertifiedInside}{RGB}{200, 140, 220}
    \definecolor{CertifiedBorder}{RGB}{150, 60, 155}

    \definecolor{ClassA}{RGB}{95, 190, 60}
    \definecolor{ClassB}{RGB}{225, 55, 55}

    \definecolor{similar}{RGB}{235, 45, 150}
    \definecolor{ball}{RGB}{235, 125, 30}

    \definecolor{myblue}{RGB}{0, 64, 255}

    \pgfmathsetmacro{\xf}{2}
    \pgfmathsetmacro{\yf}{4.5}
    \pgfmathsetmacro{\xs}{\xf}

    \draw[very thick, opacity=0.17, fill=DPinsideSmall, draw=DPborder, rounded corners=15pt] (-.25, -2) -- ++(0, 8.5) -- ++(4.5, 0) -- ++(0, -8.5) -- ++(-4.5, 0) --cycle;

    \node[star, star points=4, star point ratio=0.25, fill=black, scale=1] (zGdot) at (\xf, \yf) {};
    \node[font=\footnotesize] (zG) at (\xf + .25, \yf - .25) {$\glowpoint$};

    \node[font=\footnotesize] (zG) at (\xf - 1.2, \yf - 1.625) {$\glowpoint - \epsilon \attrvec_{\text{pale}}$};
    \node[font=\footnotesize] (zG) at (\xf + 1.2, \yf + 1.625) {$\glowpoint + \epsilon \attrvec_{\text{pale}}$};

    \draw[-|, dashed] (zGdot) -- ++(1.2, 1.2);
    \draw[-|, dashed] (zGdot) -- ++(-1.2, -1.2);

    \draw[-latex, ultra thick, draw=similar] (zGdot) -- (\xf + .6, \yf + .6);
    \node[similar] (attribute) at (\xf + 1, \yf + 0.25) {$\attrvec_{\text{pale}}$};

    \node[circle, fill=black, scale=0.4] (zg1dot) at (\xf + 0.9, \yf + 0.9) {};
    \node[circle, fill=black, scale=0.4] (zg2dot) at (\xf - 0.9, \yf - 0.9) {};
    \node[circle, fill=black, scale=0.4] (zg3dot) at (\xf - 0.5, \yf - 0.5) {};

    \node (samples) at (\xf - 0.75, \yf + 0.75) {Samples};
    \draw[dotted] (samples) -- (zg1dot);
    \draw[dotted] (samples) -- (zg2dot);
    \draw[dotted] (samples) -- (zg3dot);

    \pgfmathsetmacro{\xss}{\xs - .5}
    \pgfmathsetmacro{\yss}{4.25}

    \draw[->, thick] (\xf, \yf - 1.75) -- ++(0, -0.5);
    \node[font=\large, rectangle, draw, align=center, very thick, color=netborder, fill=netinside, text=black, execute at begin node=\setlength{\baselineskip}{1.5em}] at (\xs, \yf - 2.875) {LASSI $\lassi$};
    \draw[->, thick] (\xf, \yf - 3.5) -- ++(0, -0.5);

    \pgfmathsetmacro{\xlcifr}{\xss + 0.5}
    \pgfmathsetmacro{\ysss}{\yss - 5}

    \node (DPcenterdot) at (\xlcifr, \ysss) {};
    \filldraw[draw=Pborder, ultra thick, fill=Pinside] (DPcenterdot) circle (1);
    \node[regular polygon, regular polygon sides=3, fill=black, scale=0.4] at (DPcenterdot) {};
    \node[font=\normalsize] (DPcenter) at (\xlcifr + 0.125, \ysss + .375) {$\CScenter$};
    \draw[-, thick] (DPcenterdot) -- ++(1, 0) node[midway, below, font=\normalsize] {$\CSradius$};

    \node[star, star points=4, star point ratio=0.25, fill=black, scale=1] (zLdot) at (\xlcifr - 0.5, \ysss - 0.25) {};
    \node[font=\normalsize] (zL) at (\xlcifr - 0.5, \ysss - .5) {$\lcifrpoint$};

    \node[circle, fill=black, scale=0.4] (zl1dot) at (\xlcifr - 0.5, \ysss + 0.5) {};
    \node[circle, fill=black, scale=0.4] (zl3dot) at (\xlcifr - 0.75, \ysss) {};
    \node[circle, fill=black, scale=0.4] (zl2dot) at (\xlcifr, \ysss - 0.75) {};

    \draw[dashed] (zl1dot) to[out=-150, in=90] (zl3dot) to[out=-90, in=180] (zLdot) to[out=0, in=120] (zl2dot);

\end{tikzpicture} }
	\caption{Center smoothing the similarity set.}
	\label{fig:centersmoothing}
\end{wrapfigure}
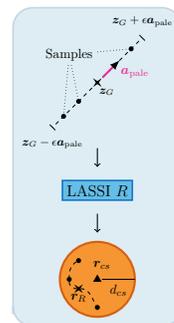

Given a point $\generativepoint$ in the latent space of $G$, we define the function
$g_\generativepoint\left(t\right) \coloneqq
\lassi\left(\generativepoint + t \cdot \attrvec\right)$ for
$t \in \mathbb{R}$.
We apply the center smoothing procedure presented by~\cite{kumar2021center} to
obtain $\widehat{g_\generativepoint}$, the smoothed version of $g_\generativepoint$,
and define
$\smoothlassi\left(\generativepoint\right) \coloneqq
\widehat{g_\generativepoint}\left(0\right)$ such that for all
$\generativepoint' \in \simset$, $\| \smoothlassi\left(\generativepoint\right) -
\smoothlassi\left(\generativepoint'\right) \|_2 \leq \CSradius$
(see~\cref{fig:centersmoothing}).
Next, we smooth the classifier $\consumer$ to obtain its
$\ell_2$-robustness radius $\RSradius$.
If $\CSradius < \RSradius$, then the end-to-end model $\model = \smoothc \circ
\smoothlassi \circ \glowenc$ certifiably satisfies individual fairness at
$\inputpoint$ (as defined in \cref{eq:ind-fairness-certificate-input-space})
with high probability.
Concretely, if we instantiate center smoothing with confidence $\CSprob$ and
randomized smoothing with confidence $\RSprob$, then the individual fairness
certificate holds with probability at least $1 - \CSprob - \RSprob$ (union
bound).
The compositional certification procedure is summarized in
\cref{alg:certify-fairness}. Its correctness is formalized in \cref{thm:endtoend}
with a detailed proof in \cref{app:theorem-proof}.

\begin{algorithm}[t]
	\caption{Certifying the individual fairness of
		$\smoothc \circ \smoothlassi \circ \glowenc$ for the input $\inputpoint$.}
	\label{alg:certify-fairness}
	\begin{algorithmic}
		\STATE \textbf{function} \textsc{Certify}($\glowenc$, $\lassi$,
		$\consumer$, $\inputpoint$)
		\STATE \quad Let $\glowpoint = \glowenc\left(\inputpoint\right)$. Then,
		$\CScenter = \smoothlassi\left(\glowpoint\right)$
		and $\CSradius$ from center smoothing~\cite{kumar2021center}.
		\STATE \quad \textbf{if} center smoothing abstained \textbf{then return}
		\textsc{Abstain}
		\STATE \quad Smooth $\consumer$~\cite{cohen2019certified}: obtain the
		certified radius $\RSradius$ around $\CScenter$ (\ie same classification)
		\STATE \quad \textbf{if} $\CSradius < \RSradius$ \textbf{then return} \textsc{Certified}
		\STATE \quad \textbf{else return} \textsc{Not Certified}
	\end{algorithmic}
\end{algorithm}

\begin{restatable}{theorem}{endtoend}
	Assume that we have a bijective generative model $G = (\glowenc, \glowdec)$
	used to define the similarity set $\simseti$ for a given input $\inputpoint$.
	Let~\cref{alg:certify-fairness} perform center
    smoothing~\cite{kumar2021center} with confidence $1 - \CSprob$ and
    randomized smoothing~\cite{cohen2019certified} with confidence $1 - \RSprob$.
	If~\cref{alg:certify-fairness} returns \textsc{Certified} for the input
    $\inputpoint$, then the end-to-end model
    $M = \smoothc \circ \smoothlassi \circ \glowenc$ is individually fair for
    $\inputpoint$ with respect to $\simseti$ with probability at least
    $1 - \CSprob - \RSprob$.
	\label{thm:endtoend}
\end{restatable}

\section{Experiments}
\label{sec:experimental-evaluation}

We now evaluate LASSI and present the key findings: (i)
LASSI enforces individual fairness and keeps accuracy high, (ii) LASSI handles
various sensitive attributes and attribute vectors, and (iii) LASSI
representations transfer to unseen tasks.

\medskip

\noindent\textbf{Datasets} {}
We evaluate LASSI on two datasets.
CelebA~\cite{liu2015deep} contains \num[group-separator={,}]{202599}
aligned and cropped face images of real-world celebrities. The images are
annotated with the presence or absence of $40$ face attributes with various
correlations between them~\cite{denton2019detecting}.
As CelebA is highly imbalanced, we also experiment with
FairFace~\cite{karkkainen2021fairface}. It is balanced on race and contains
\num[group-separator={,}]{97698} released images (padding $0.25$) of individuals
from $7$ race and $9$ age groups.
We split the training set randomly ($80$:$20$ ratio) and
evaluate on the validation set because the test set is not publicly shared.
Further information about the datasets (including experimental
``unfairness'' of different attributes computed on CelebA)
is in~\cref{app:datasets}.

\medskip

\noindent\textbf{Experimental setup} {}
The following setup is used for all experiments, unless stated otherwise.
We use images of size $64$$\times$$64$, and
for each dataset pretrain a Glow model $G$ with $4$ blocks of $32$ flows,
using an open-source PyTorch~\cite{paszke2019pytorch}
implementation~\cite{seonghyeon2020glow}. We use
$\attrvec = \vz_{G, pos} - \vz_{G, neg}$ and set
$\epsilon = 1$ such that
$\simseti$ contains realistic high-quality reconstructions (confirmed by
manual inspection).
Thus, the similarity specification (\cref{sec:similarity}) for enforcing
individual fairness is determined by $G$ and the radius $\epsilon$.
We implement the representation $\lassi$ as a fully-connected
network that propagates Glow's latent code of an input
$\inputpoint$ through two hidden layers of sizes 2048 and 1024,
mapping to a 512-dimensional space.
The final layer applies zero mean
and unit variance normalization ensuring that all components of $\lassi$'s output
are in the same range when Gaussian noise is added during smoothing.
A linear classifier $\consumer$ is used for predicting the target label.

Our fairness-unaware baseline (denoted as Naive) is standard
representation learning of $\lassi$ without adversarial and reconstruction losses
($\lambda_2 = \lambda_3 = 0$). When training LASSI, we set the
classification loss weight $\lambda_1 = 1$, except for the transfer learning
experiments.
A recent work~\cite{ramaswamy2021fair} proposed generating synthetic images
with a ProGAN~\cite{karras2018progressive} to balance the dataset.
Their method is not concerned with individual fairness and their transformation
of latent representations may change other, non-sensitive attributes.
Nevertheless, we employ~\cite{ramaswamy2021fair}'s high-level idea of augmenting
the training set with synthetic samples from a generative model (Glow in our case).
For each training sample $\inputpoint$, we
synthesize and randomly sample $s$ additional images from $\simseti$ in every epoch.
Then, we proceed with representation learning of
$\lassi$ on the augmented dataset.
We denote this baseline, addapted to the individual fairness setting, as DataAug.
We do not compare with LCIFR~\cite{ruoss2020learning} as our individual similarity
specifications cannot be directly encoded as logical formulas over the input
features of $\inputpoint$ and because its certification is based on
expensive solvers that do not scale to Glow and large models.

We list all selected hyperparameters for all experiments, based on an
an extensive hyperparameter search on the validation sets,
in~\cref{app:hyperparam-tuning} (details provided for the CelebA dataset).
The hyperparameter study shows that LASSI works for a wide range of
hyperparameter values and demonstrates that $\lambda_2$ controls the trade-off
between accuracy and fairness. We report the accuracy and
the certified individual fairness of the models measured on 312 samples from
CelebA's test set (every 64-th) and 343 samples from FairFace's test set
(every 32-nd). The certified fairness refers to the percentage of test samples for
which~\cref{alg:certify-fairness} returns \textsc{Certified}, \ie for which we
can prove that~\cref{eq:ind-fairness-certificate-input-space} holds, guaranteeing
that all similar individuals (according to our similarity definition) are classified
the same. This metric is denoted as ``Fair'' in the tables.
The evaluation of a single data point takes up to 6 seconds due to the sampling
required by the smoothing procedures, which is why we do not report results on
the whole test sets.
We ran the experiments on GeForce RTX 2080 Ti GPUs
and release all the code and models to reproduce our results at
\url{https://github.com/eth-sri/lassi}.

\begin{figure}[tb]
    \centering
    \subfloat[\texttt{Pale\_Skin}]{
		\includegraphics[width=0.45\linewidth]{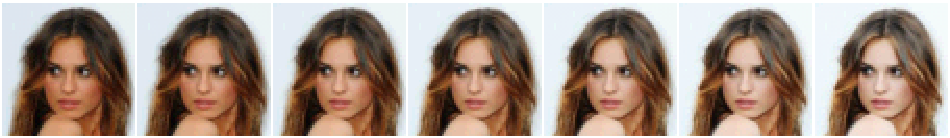}
    }\hfil
    \subfloat[\texttt{Young}]{
	    \includegraphics[width=0.45\linewidth]{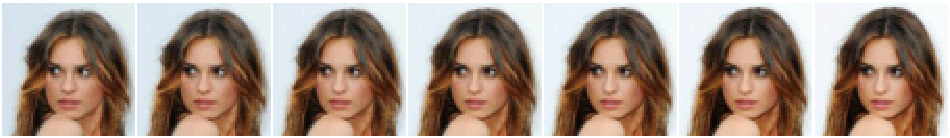}
    }

    \subfloat[\texttt{Blond\_Hair}]{
	    \includegraphics[width=0.45\linewidth]{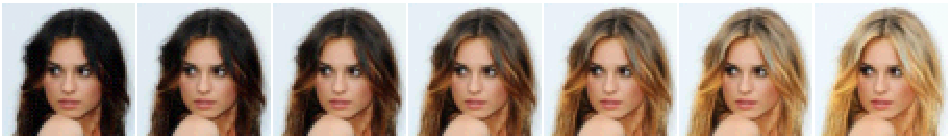}
    }\hfil
    \subfloat[\texttt{Heavy\_Makeup}]{
	    \includegraphics[width=0.45\linewidth]{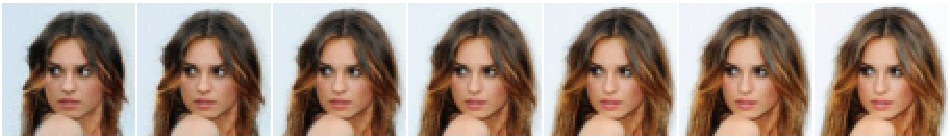}
    }

    \subfloat[\texttt{Pale\_Skin + Young}]{
	    \includegraphics[width=0.45\linewidth]{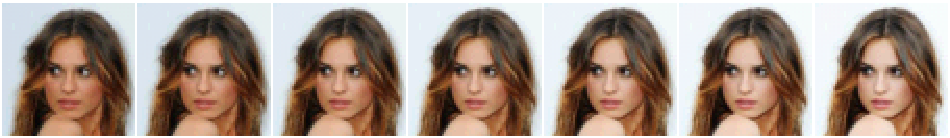}
    }\hfil
    \subfloat[\texttt{Pale\_Sking + Young + Blond}]{
	    \includegraphics[width=0.45\linewidth]{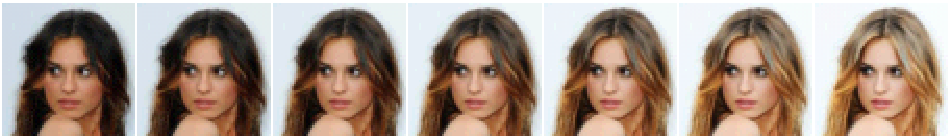}
    }

    \caption{
        Similar points from $\simseti$,
		as reconstructed by Glow, for multiple sensitive attribute
		combinations. Central images correspond to
		the original input. We vary $t$ uniformly (left to right) in
		the $[-\frac{\epsilon}{\sqrt{n}}, \frac{\epsilon}{\sqrt{n}}]$
        range, $n =$ number of sensitive attributes, $\epsilon = 1$. For $n > 1$,
        all attribute vectors are multiplied by the same $t$.
    }
	\label{fig:similarity-set-main-text}
\end{figure}

\begin{table*}[tb]
    \caption{
        Evaluation of LASSI on the CelebA dataset, showing that LASSI significantly
        increases certified individual fairness compared to the baselines without
        affecting the classification accuracy, even increasing it for imbalanced
        tasks. Reported means averaged over $5$ runs,
        see~\cref{app:more-experimental-results-stddev} for standard deviations.
    }
    \begin{center}
    \resizebox{0.95\linewidth}{!}{\begin{tabular}{@{}lP{0.03cm}lcP{1cm}P{1cm}cP{1cm}P{1cm}cP{1cm}P{1cm}@{}}
    \toprule
    & & & & \multicolumn{2}{c}{Naive} & & \multicolumn{2}{c}{DataAug} & & \multicolumn{2}{c}{LASSI (ours)} \\
    \cmidrule{5-6}
    \cmidrule{8-9}
    \cmidrule{11-12}
    Task & & Sensitive attribute(s) & & Acc & Fair && Acc & Fair && Acc & Fair \\
    \midrule
    \multirow{6}{*}{\texttt{Smiling}} && \texttt{Pale\_Skin} & & \textbf{86.3} & 0.6 && 85.7 & 12.2 && 85.9 & \textbf{98.0} \\
    & & \texttt{Young} & & \textbf{86.3} & 38.2 && 85.9 & 43.0 && \textbf{86.3} & \textbf{98.8} \\
    & & \texttt{Blond\_Hair} & & 86.3 & 3.4 && \textbf{86.6} & 9.4 && 86.4 & \textbf{94.7} \\
    & & \texttt{Heavy\_Makeup} & & \textbf{86.3} & 0.4 && 85.3 & 13.7 && 85.6 & \textbf{91.3} \\
    & & \texttt{Pale+Young} & & \textbf{86.0} & 0.4 && 85.8 & 9.9 && 85.8 & \textbf{97.3} \\
    & & \texttt{Pale+Young+Blond} & & 86.2 & 0.0 && \textbf{86.4} & 3.6 && 85.5 & \textbf{86.5} \\
    \midrule
    \multirow{4}{*}{\texttt{Earrings}} && \texttt{Pale\_Skin} & & 81.3 & 24.3 && 81.0 & 40.4 && \textbf{85.0} & \textbf{98.5} \\
    & & \texttt{Young} & & 81.4 & 59.2 && 79.9 & 72.0 && \textbf{84.5} & \textbf{98.0} \\
    & & \texttt{Blond\_Hair} & & 81.4 & 9.2 && 82.2 & 30.5 && \textbf{84.8} & \textbf{96.2} \\
    & & \texttt{Heavy\_Makeup} & & 81.6 & 20.5 && 80.3 & 49.2 && \textbf{82.3} & \textbf{98.7} \\
    \bottomrule
\end{tabular}
}
    \end{center}
    \label{tbl:results}
\end{table*}

\medskip

\noindent\textbf{Single sensitive attribute} {}
We experiment with $4$ different continuous sensitive attributes from CelebA:
\texttt{Pale\_Skin}, \texttt{Young}, \texttt{Blond\_Hair} and
\texttt{Heavy\_Makeup} on two tasks: predicting \texttt{Smiling} and
\texttt{Earrings}. We chose attributes with different balance ratios
that have been used in prior work~\cite{denton2019detecting}, while
avoiding attributes that perpetuate harmful stereotypes~\cite{denton2019detecting}
(\eg avoiding \texttt{Male}). Glow can also be used to generate discrete
attributes, but then fairness certification can be done via
enumeration because partial eyeglasses or hats, for example, are not plausible.
\cref{fig:similarity-set-main-text} provides example images from
$\simseti$ for a single $\inputpoint$.
The \texttt{Earrings} task is considerably more imbalanced than
\texttt{Smiling}, with $78.21$\% majority class accuracy on our
test subset. Because of the high correlation between
\texttt{Earrings} and \texttt{Makeup}, we run LASSI with increased $\lambda_2$ for
this pair of attributes.

We show the results in~\cref{tbl:results} averaged over $5$ runs with different
random seeds.
The results indicate that data augmentation helps, but is not enough.
LASSI significantly improves the certified fairness, compared to the baselines,
with a minor loss of accuracy on \texttt{Smiling} and even acts as a helpful
regularizer on the imbalanced \texttt{Earrings} task.
In~\cref{app:more-experimental-results-stddev} we report the standard deviations
demonstrating that LASSI consistently enforces individual fairness with low
variance and further evaluate empirical (\ie non-certifiable) fairness metrics.

\medskip

\noindent\textbf{Multiple sensitive attributes} {}
In the next experiment, we combine the sensitive attributes
\texttt{Pale\_Skin}, \texttt{Young} and \texttt{Blond\_Hair} and predict
\texttt{Smiling}. The similarity sets \wrt which we certify individual fairness
are defined as
$\simsetl = \{
\glowenc\left(\inputpoint\right)
+ \sum_{i} t_i \cdot \attrvec_i
\mid
\|\vt\|_2 \leq \epsilon\}$.
The results in~\cref{tbl:results} (rows $5$ -- $6$) show that the
certified fairness drops as the similarity sets become more complex, as
expected, but LASSI still successfully enforces individual fairness in
these cases.

\medskip

\noindent\textbf{Larger images and different attribute vectors} {}
Next, we explore if LASSI can also work with larger images. We increase the
dimensionality of the CelebA images to $128$$\times$$128$, pretrain Glow with 5
blocks and keep the rest of the hyperparameters the same. The results are
consistent with those already presented in~\cref{tbl:results}: LASSI
increases the certified individual fairness by up to 77\% on the \texttt{Smiling}
task (see~\cref{app:more-experimental-results-stddev} for detailed results).
We also instantiate LASSI with the alternative attribute vector
type~\cite{denton2019detecting}
introduced in~\cref{sec:similarity} (with $\epsilon = 10$). Although interpolating
along the vector which is perpendicular to the linear decision boundary of the sensitive
attribute
possibly reduces the correlations leaked into the similarity sets,
\cref{tbl:denton-attrib-vec} shows that
LASSI still improves the certified fairness by up to 16\% compared to the
baselines. This improvement is 9.7\% and 6.1\% for the attribute vectors
proposed by~\cite{ramaswamy2021fair} and~\cite{li2021discover} respectively,
further demonstrating that LASSI can be useful for various attribute vector types.
More details about these experiments are provided
in~\cref{app:more-attrib-vectors}.

\begin{table*}[tb]
	\caption{
        Evaluation with $\attrvec$ perpendicular to the
        linear decision boundary of the sensitive attribute~\cite{denton2019detecting}
        (\cref{sec:similarity})
        on the \texttt{Smiling} task, showing that LASSI is not limited to
        a specific attribute vector type.
	}
	\begin{center}
		\resizebox{0.85\linewidth}{!}{\begin{tabular}{@{}lcP{1cm}P{1cm}cP{1cm}P{1cm}cP{1cm}P{1cm}@{}}
    \toprule
    & & \multicolumn{2}{c}{Naive} && \multicolumn{2}{c}{DataAug} && \multicolumn{2}{c}{LASSI (ours)} \\
    \cmidrule{3-4}
    \cmidrule{6-7}
    \cmidrule{9-10}
    Sensitive attribute(s) & & Acc & Fair && Acc & Fair && Acc & Fair \\
    \midrule
    \texttt{Pale\_Skin} & & 86.4 & 34.0 && 85.9 & 90.3 && \textbf{86.5} & \textbf{98.8} \\
    \texttt{Young} & & 86.3 & 73.1 && 86.2 & 90.3 && \textbf{86.8} & \textbf{97.9} \\
    \texttt{Blond\_Hair} & & 86.2 & 71.4 && 86.1 & 88.8 && \textbf{86.7} & \textbf{98.8} \\
    \texttt{Heavy\_Makeup} & & 86.2 & 11.5 && 86.3 & 87.4 && \textbf{86.8} & \textbf{98.8} \\
    \texttt{Pale+Young} & & 86.2 & 28.6 && 85.8 & 84.7 && \textbf{86.5} & \textbf{98.6} \\
    \texttt{Pale+Young+Blond} & & 86.2 & 23.7 && 85.9 & 82.2 && \textbf{86.4} & \textbf{98.7} \\
    \bottomrule
\end{tabular}
}
	\end{center}
	\label{tbl:denton-attrib-vec}
\end{table*}

\medskip

\noindent\textbf{Transfer learning} {}
To demonstrate the modularity of our approach, we show that LASSI can learn fair
and transferable representations
which are useful for unseen downstream tasks. To that end, we turn off
the classification loss, consistent with prior work~\cite{madras2018learning}
($\lambda_1 = 0$, \ie the representation $\lassi$ is
trained unsupervised), and enable the reconstruction loss ($\lambda_3 = 0.1$).
The reconstruction network $Q$ has an architecture symmetric to that of $\lassi$.
In~\cref{tbl:transfer} we report the accuracies and the certified fairness on
$7$ different, relatively well-balanced, downstream tasks.
The models perform slightly worse compared to the case where the downstream task
is known in advance, but the obtained certified individual fairness is still
consistently high -- more than 80\% for the most complex similarity specification
(\texttt{P+Y+B})
and above 90\% for the simpler ones. Standard deviations and baseline accuracies on
these tasks are reported in~\cref{app:more-experimental-results-stddev}.

\begin{table*}[tb]
    \caption{
        Transfer learning results, demonstrating that LASSI can still achieve
        high certified individual fairness even when the downstream tasks are not
        known.
    }
    \begin{center}
    \resizebox{\linewidth}{!}{\begin{tabular}{@{}lcP{0.75cm}cP{0.75cm}ccP{0.75cm}cP{0.75cm}ccP{0.75cm}cP{0.75cm}ccP{0.75cm}cP{0.75cm}ccP{0.75cm}cP{0.75cm}@{}}
    \toprule
    Sens. attrib.: & & \multicolumn{3}{c}{\texttt{Pale} (\texttt{P})} &&& \multicolumn{3}{c}{\texttt{Young} (\texttt{Y})} &&& \multicolumn{3}{c}{\texttt{Blond} (\texttt{B})} &&& \multicolumn{3}{c}{\texttt{P + Y}} &&& \multicolumn{3}{c}{\texttt{P + Y + B}}\\
    \cmidrule{3-5}
    \cmidrule{8-10}
    \cmidrule{13-15}
    \cmidrule{18-20}
    \cmidrule{23-25}
    Transfer task & & Acc && Fair &&& Acc && Fair &&& Acc && Fair &&& Acc && Fair &&& Acc && Fair \\
    \midrule
    \texttt{Smiling} & & 86.2 && 93.1 &&& 86.0 && 95.4 &&& 85.1 && 93.8 &&& 85.9 && 92.2 &&& 85.1 && 87.0 \\
    \texttt{High\_Cheeks} & & 81.7 && 92.6 &&& 82.3 && 96.0 &&& 81.3 && 92.2 &&& 80.8 && 93.0 &&& 80.6 && 84.5 \\
    \texttt{Mouth\_Open} & & 81.5 && 91.2 &&& 82.4 && 94.3 &&& 82.4 && 87.5 &&& 81.6 && 90.1 &&& 82.5 && 80.8 \\
    \texttt{Lipstick} & & 88.3 && 94.0 &&& 85.8 && 95.8 &&& 86.8 && 91.2 &&& 85.1 && 90.6 &&& 86.2 && 81.0 \\
    \texttt{Heavy\_Makeup} & & 86.5 && 93.0 &&& 83.5 && 95.3 &&& 85.6 && 89.3 &&& 83.7 && 90.0 &&& 83.3 && 80.4 \\
    \texttt{Wavy\_Hair} & & 79.2 && 93.3 &&& 77.5 && 95.8 &&& 78.0 && 91.3 &&& 77.6 && 91.5 &&& 78.8 && 85.3 \\
    \texttt{Eyebrows} & & 78.3 && 92.1 &&& 78.3 && 94.7 &&& 78.9 && 89.6 &&& 77.8 && 92.2 &&& 78.7 && 85.6 \\
    \bottomrule
\end{tabular}
}
    \end{center}
    \label{tbl:transfer}
\end{table*}

\medskip

\noindent\textbf{Training on FairFace dataset} {}
To verify that LASSI works well in different settings,
we also evaluate on the FairFace~\cite{karkkainen2021fairface} dataset.
We select \texttt{Race=Black} as a sensitive attribute and predict \texttt{Age}.
This is a very challenging multi-class task with around $60$\% state of the art
accuracy. Therefore, we create two easier tasks: \texttt{Age-2}, predicting if
an individual is younger or older than $30$, and \texttt{Age-3} with three target ranges:
$[0-19]$, $[20-39]$, and $40+$. \cref{tbl:fairface-main-text} reports the
results for $\epsilon=0.5$.
We verify that transfer learning also works
in this setup by training on \texttt{Age-2} and then transferring the representations
to all three tasks. As the tasks are related, increasing the classification loss
weight $\lambda_1$ on the base task from $0$ to $0.01$,
increases both the
transfer downstream accuracy and the certified fairness. The highest certified
fairness is generally obtained when the downstream task is known and the model is
trained on it (LASSI, $\lambda_1 = 1$).

\begin{table*}[tb]
    \caption{
        Results on FairFace, showing that LASSI can significantly
        improve the certified individual fairness even on balanced datasets.
        The adversarial loss weight is $\lambda_2 = 0.1$ for all
        models except Naive, the
        transfer models are trained on \texttt{Age-2}
        with reconstruction loss weight $\lambda_3 = 0.1$. LASSI is
        trained on the corresponding tasks with adversarial but without
        reconstruction loss
        ($\lambda_1 = 1$, $\lambda_3 = 0$).
    }
    \begin{center}
    \resizebox{0.9\linewidth}{!}{\begin{tabular}{@{}lcccccccccccccccccccccccc@{}}
    \toprule
    & & \multicolumn{3}{c}{Naive} &&& \multicolumn{3}{c}{DataAug} &&& \multicolumn{3}{c}{$\text{Transfer}_{\lambda_1=0}$} &&& \multicolumn{3}{c}{$\text{Transfer}_{\lambda_1=0.01}$} &&& \multicolumn{3}{c}{LASSI}\\
    \cmidrule{3-5}
    \cmidrule{8-10}
    \cmidrule{13-15}
    \cmidrule{18-20}
    \cmidrule{23-25}
    Task & & Acc && Fair &&& Acc && Fair &&& Acc && Fair &&& Acc && Fair &&& Acc && Fair \\
    \midrule
    \texttt{Age-2} & & 69.0 && 5.7 &&& 68.9 && 4.8 &&& 66.4 && 91.7 &&& \textbf{74.9} && 91.7 &&& 72.0 && \textbf{95.0} \\
    \texttt{Age-3} & & 67.0 && 0.0 &&& 67.1 && 0.6 &&& 63.0 && 85.6 &&& \textbf{67.7} && 88.0 &&& 65.1 && \textbf{90.8} \\
    \texttt{Age} (all) & & \textbf{42.2} && 0.0 &&& 39.9 && 0.0 &&& 34.3 && 72.0 &&& 37.1 && \textbf{77.5} &&& 41.5 && 65.9 \\
    \bottomrule
\end{tabular}
}
    \end{center}
    \label{tbl:fairface-main-text}
\end{table*}

\section{Limitations and Future Work}

We now discuss some of the limitations of LASSI.
First, our method trains individually fair models, but it does not guarantee
that models satisfy other fairness notions, \eg group fairness.
While individual fairness is a well-studied research area, recent work argues
that it does not qualify as a valid fairness notion as it can be insufficient to
guarantee fairness in certain instances and risks encoding implicit human
biases~\cite{fleisher2021whats}.
Moreover, the validity of our fairness certificates depends heavily on the
generative model used by LASSI.
In particular, the similarity sets $\simset$ considered in our work may not be
exhaustive enough as there can be latent points outside $\simset$ that
correspond to input points that would be perceived as similar to $\inputpoint$
by a human observer.
This can also happen if the generative model is not powerful enough to generate
all possible instances and combinations of similar individuals.
For the above reasons, it is hard to obtain formal guarantees about $G$ and the
computed certificates may not always transfer from $G$ to the real world.
We explore this issue further in~\cref{app:3dshapes} where we experiment with
3D Shapes~\cite{burgess2018shapes}, a procedurally generated dataset
with known ground truth similarity sets.
Future work can consider addressing these challanges by performing extensive manual
human inspection of reconstructions produced by $G$ (similar to
\cref{app:glow-reconstructions}).
Moreover, all future advancements in the active research area of normalizing
flows will immediately improve the quality of our certificates.

\section{Conclusion}
\label{sec:conclusion}

We proposed LASSI, which defines image similarity with respect to a generative model via
attribute manipulation, allowing us to capture complex image transformations
such as changing the age or skin color, which are
otherwise difficult to characterize. Further, we were able to scale certified
representation learning for individual fairness to real-world high-dimensional
datasets by using randomized smoothing-based techniques. Our extensive
evaluation yields promising results on several datasets and illustrates the practicality of our
approach.

\medskip

\noindent\textbf{Acknowledgments} {}
We thank Seyedmorteza Sadat for his help with preliminary investigations and the
anonymous reviewers for their insightful feedback.

\message{^^JLASTBODYPAGE \thepage^^J}

\bibliographystyle{splncs04}
\bibliography{references}

\message{^^JLASTREFERENCESPAGE \thepage^^J}

\ifbool{includeappendix}{%
	\clearpage
	\appendix
	\section*{Ethics Statement}
\label{sec:ethics-statement}

This work proposed a novel method for certifying the individual
fairness of models operating on high-dimensional data. Progress on this challenging
problem could enable fairness auditing for high-risk computer vision applications,
such as facial recognition. Recent work~\cite{stark2019facial}
argues that facial recognition algorithms can have undesirable,
socially toxic, and divisive consequences. For instance, it was demonstrated
that they may perpetuate and reinforce racial and gender
bias~\cite{buolamwini2018gender,denton2019detecting}. Therefore, they
must be applied carefully, considering the social dynamics and context
in which they occur. Accordingly, following prior work~\cite{denton2019detecting},
we refrained from using unstable social constructs, such as gender, or normatively
judgemental attributes, such as ``attractive'' or ``chubby'', in our research.

One way to limit the potential harms of facial analysis technologies is to control
and regulate their usage. Our work aims to help fill this gap by presenting a
methodology for enforcing individual fairness via certification. As highlighted in
our paper, we acknowledge that the quality of the generative models is a significant
bottleneck of our certificates. For example, they may encode various biases present in
the data. Another possible source of bias is the human perception and social
constructs which can potentially impact the validity of our similarity specifications.
Nevertheless, we believe that we can still leverage generative models and their
latent space to construct more meaningful individual fairness specifications on
high-dimensional data than those allowed by prior work. More broadly,
developing rigorous, standardized processes for auditing and certifying facial
recognition models (including human inspection, \eg by considering the reconstructed
images in~\cref{app:glow-reconstructions}) should complement the contributions
presented in our work. Finally, future quality advancements in generative modelling
and normalizing flows can directly translate into stronger guarantees of our method,
enabling certified fair application of models using rich, high-dimensional data.

\section{Proof of~\cref{thm:endtoend}}
\label{app:theorem-proof}

This section provides a formal proof of the following:

\endtoend*

\noindent To prove~\cref{thm:endtoend}, we will make use of the following
randomized and center smoothing theorems proved in the literature:

\begin{theorem}[Adapted from \cite{cohen2019certified}]
	Let $\consumer \colon \representationspace \to \outputspace$ be a classifier
	and let $\bm{\varepsilon} \sim \mathcal{N}(0, \RSsigma^2 I)$. Let $\smoothc$
	be defined such that $\smoothc\left(\representationpoint\right) =
	\argmax_{c \in \outputspace}
	\prob_{\bm{\varepsilon}}(\consumer(\representationpoint + \bm{\varepsilon}) = c)$.
    Suppose $c_A \in \mathcal{Y}$ and
	$\underline{p_A}, \overline{p_B} \in [0,1]$ satisfy:
    \begin{equation}
        \prob_{\bm{\varepsilon}}
        (\consumer(\representationpoint + \bm{\varepsilon}) = c_A)
        \geq
        \underline{p_A}
        \geq
        \overline{p_B}
        \geq
        \max_{c_B \neq c_A}
		\prob_{\bm{\varepsilon}}
		(\consumer(\representationpoint + \bm{\varepsilon}) = c_B).
    \end{equation}
    Then $\smoothc(\representationpoint + \bm{\delta}) = c_A$ for all $\bm{\delta}$
	satisfying $\| \bm{\delta} \|_2 < \RSradius$, where
	$\RSradius \coloneqq
	\tfrac{\RSsigma}{2}(\Phi^{-1}(\underline{p_A}) - \Phi^{-1}(\overline{p_B}))$.
    \label{thm:cohen-smoothing}
\end{theorem}

\noindent Here, $\outputspace$ denotes the set of class labels, $\Phi$ is the
cumulative distribution function (CDF) of the standard normal distribution
$\mathcal{N}(0, 1)$, and $\Phi^{-1}$ is its inverse.

\begin{theorem}[Adapted from \cite{kumar2021center}]
	Let $g \colon \mathbb{R}^a \to \representationspace$ and
	$\hat{g} \colon \mathbb{R}^a \to \representationspace$ is an approximation
	of the smoothed version of $g$, which maps $\vt \in \mathbb{R}^a$ to the center
	point $\hat{g}\left(\vt\right)$ of a minimum enclosing
	ball containing at least half of the points
	$\representationpoint_i \sim g(\vt + \mathcal{N}(0, \CSsigma^2 I))$,
	$i \in \{1, \ldots, m\}$. Then, for $\epsilon > 0$, with probability at least
	$1 - \CSprob$ we have,
    \begin{equation}
        \forall \vt' \text{ s.t. }
            \|\vt - \vt'\|_2 \leq \epsilon,
            \|\hat{g}(\vt) - \hat{g}(\vt')\|_2 \leq \CSradius.
    \end{equation}
\label{thm:center-smoothing}
\end{theorem}

\noindent We now proceed to proving~\cref{thm:endtoend}:

\begin{proof}
	Assume that~\cref{alg:certify-fairness} returns \textsc{Certified} for the
	input $\inputpoint$. We need to show that with probability at least $1 - \CSprob - \RSprob$
	\begin{equation}
		\forall \inputpoint' \in \simseti:
		M\left(\inputpoint\right) = M\left(\inputpoint'\right),
		\tag{Eq. \ref{eq:ind-fairness-certificate-input-space}}
	\end{equation}
	where $\model = \smoothc \circ \smoothlassi \circ \glowenc$.
	By the definition of $\simseti$ and $\glowenc$ being the inverse of
	$\glowdec$, we have for all $\inputpoint' \in \simseti$, $\generativepoint'
	= \glowenc(\inputpoint') \in \simset$, hence it suffices to prove
	\begin{equation}
		\forall \generativepoint' \in \simset:
		\smoothc \circ \smoothlassi \left(\glowpoint\right) =
		\smoothc \circ \smoothlassi \left(\generativepoint'\right),
		\label{eq:ind-fairness-certificate-latent-space}
	\end{equation}
	where $\glowpoint = \glowenc\left(\inputpoint\right)$.

	Next, recall the definition of
	$g_\generativepoint\left(t\right) \coloneqq
	\lassi\left(\generativepoint + t \cdot \attrvec\right)$
	and note that for
	$\generativepoint' = \generativepoint + t' \cdot \attrvec$,
	the center smoothing of
	\begin{align*}
		\widehat{g_{\generativepoint'}}\left(t\right)&
		\text{: samples from }
		g_{\generativepoint'}\left(t + \mathcal{N}(0, \CSsigma^2)\right)
		= \lassi\left(
			\generativepoint' + \left(t + \mathcal{N}(0, \CSsigma^2)\right)
			\cdot \attrvec
		\right)\text{;}
		\\
		\widehat{g_{\generativepoint}}\left(t + t'\right)&
		\text{: samples from }
		g_{\generativepoint}\left(t + t' + \mathcal{N}(0, \CSsigma^2)\right)
		= \lassi\left(
			\generativepoint + \left(t + t' + \mathcal{N}(0, \CSsigma^2)\right)
			\cdot \attrvec
		\right).
	\end{align*}
	Since $\generativepoint' = \generativepoint + t' \cdot \attrvec$, the sampling
	distributions are the same, hence
	$\widehat{g_{\generativepoint'}}\left(t\right) =
	\widehat{g_{\generativepoint}}\left(t + t'\right)$, and in particular
	$\smoothlassi\left(\generativepoint'\right) =
	\widehat{g_{\generativepoint'}}\left(0\right) =
	\widehat{g_{\generativepoint}}\left(t'\right)$.

	Now, let us get back to~\cref{eq:ind-fairness-certificate-latent-space}.
	By definition of $\simset$, for all $\generativepoint' \in \simset$,
	$\generativepoint' = \glowpoint + t' \cdot \attrvec$ for some
	$t' \in \left[-\epsilon, \epsilon\right]$. Moreover,
	$\CScenter = \smoothlassi\left(\glowpoint\right)
	= \widehat{g_{\glowpoint}}\left(0\right)$ and
	$\smoothlassi\left(\generativepoint'\right) =
	\widehat{g_{\glowpoint}}\left(t'\right)$.
	\cref{thm:center-smoothing} tells us that with probability at least $1 - \CSprob$
	\begin{equation}
		\begin{aligned}
			&\forall t' \in \left[-\epsilon, \epsilon\right]. \text{ }
			\|
			\widehat{g_{\glowpoint}}\left(0\right) -
			\widehat{g_{\glowpoint}}\left(t'\right)
			\|_2 \leq \CSradius\\
			\iff&\forall \generativepoint' \in \simset. \text{ }
			\|
			\CScenter -
			\smoothlassi\left(\generativepoint'\right)
			\|_2 \leq \CSradius,
		\end{aligned}
		\label{eq:center-smoothing-ineq}
	\end{equation}
	provided that the center smoothing computation of $\CScenter$ does not
	abstain.

	Finally, we consider the last component of the pipeline -- the smoothed
	classifier $\smoothc$. Provided that $\smoothc$ does not abstain at the
	input $\CScenter$,~\cref{thm:cohen-smoothing} provides us with a radius $\RSradius$
	around $\CScenter$ such that with probability at least $1 - \RSprob$
	\begin{equation}
		\begin{aligned}
			&\forall \bm{\delta} \text{ s.t. } \|\bm{\delta}\|_2 < \RSradius, \text{ }
			\smoothc\left(\CScenter\right) =
			\smoothc\left(\CScenter + \bm{\delta}\right)\\
			\iff&\forall \vr' \text{ s.t. } \|\CScenter - \vr'\|_2 < \RSradius, \text{ }
			\smoothc\left(\CScenter\right) =
			\smoothc\left(\vr'\right).
		\end{aligned}
		\label{eq:cohen-smoothing-ineq}
	\end{equation}
	If~\cref{alg:certify-fairness} returns \textsc{Certified}, that is
	$\CSradius < \RSradius$, combining \cref{eq:center-smoothing-ineq} and
	(\ref{eq:cohen-smoothing-ineq}) and applying the union bound, we obtain that
	with probability at least $1 - \CSprob - \RSprob$ we have
	$\smoothc\left(\CScenter\right) =
	\smoothc\left(\smoothlassi\left(\vz'\right)\right)$ for all
	$\generativepoint' \in \simset$.
	That is,
	\begin{equation}
		\forall \generativepoint' \in \simset:
		\smoothc \circ \smoothlassi\left(\glowpoint\right) =
		\smoothc \circ \smoothlassi\left(\generativepoint'\right),
	\end{equation}
	as required by \cref{eq:ind-fairness-certificate-latent-space}.
	The same proof technique can also be extended to the multiple attribute vectors
	case.
	\qed
\end{proof}

\section{Datasets and Dataset Statistics}
\label{app:datasets}

In this section we provide further information and statistics about the datasets
used in this work.
CelebA\footnote{\url{https://mmlab.ie.cuhk.edu.hk/projects/CelebA.html}}~\cite{liu2015deep}
is restricted to non-commercial research and education purposes and
its authors \cite{liu2015deep} do not own the copyrights.
FairFace~\cite{karkkainen2021fairface} is licensed under CC BY 4.0.
\cref{tbl:appendix-sens-attributes} contains statistics about the sensitive
attributes and their corresponding attribute vectors.
The lengths of the CelebA attribute vectors are computed for 64$\times$64 images.

In~\cref{tbl:appendix-base-acc} we report the base accuracies of two standard
classifiers trained on the \texttt{Smiling} and \texttt{Earrings} CelebA tasks.
The first classifier is a ResNet-18 network trained directly on the original images.
The other one is a fully connected network operating on their Glow latent
representations, $\glowpoint = \glowenc\left(\inputpoint\right)$. We remark
that none of these classifiers involves representation learning.
We report the means and standard deviations, averaged over 5 runs with different
random seeds, on the validation and test sets, where the test set is the same
subset on which we report the results in the main paper. The base accuracies on
the downstream tasks used for the transfer learning experiments are reported
in~\cref{app:more-experimental-results-stddev}.

In order to estimate the relative ``unfairness'' associated with each
sensitive attribute, in~\cref{tbl:appendix-empirical-fair} we compute the
empirical individual fairness of the two classifiers. For each data point
$\inputpoint$, we sample 9 points from $\simseti$ evenly (15 points for
\texttt{Pale+Young+Blond}). If all samples are classified the same, we
add the original data point $\inputpoint$ to the empirical fairness counter.
Note that this procedure cannot certify that all points from $\simseti$ are
classified the same. Therefore, these results come with no provable guarantees
and serve as upper bounds of the certified individual fairness of the classifiers.

\begin{table*}[hbt!]
    \caption{
        Sensitive attribute statistics. The positive and negative sample ratio
		is reported for the training set, as the attribute vectors are
		computed on it.
    }
    \begin{center}
    \resizebox{0.825\linewidth}{!}{\begin{tabular}{@{}lP{0.03cm}lcccP{0.03cm}c@{}}
    \toprule
    Dataset & & Sensitive attribute & & Pos (\%) & Neg (\%) && $\|\vz_{G, pos} - \vz_{G, neg}\|_2$ \\
    \midrule
    \multirow{4}{*}{CelebA} && \texttt{Pale\_Skin} & & 4.3 & 95.7 && 11.5 \\
    & & \texttt{Young} & & 77.9 & 22.1 && 7.8 \\
    & & \texttt{Blond\_Hair} & & 14.9 & 85.1 && 15.8 \\
    & & \texttt{Heavy\_Makeup} & & 38.4 & 61.6 && 11.9 \\
    \midrule
    FairFace && \texttt{Race=Black} & & 14.1 & 85.9 && 10.9 \\
    \bottomrule
\end{tabular}
}
    \end{center}
    \label{tbl:appendix-sens-attributes}
\end{table*}

\begin{table*}[hbt!]
    \caption{
		Baseline accuracies for the \texttt{Smiling} and \texttt{Earrings} CelebA
		tasks. The ResNet-18 classifier takes the original images as an input,
		while the $\glowpoint$ classifier is a fully connected network
		classifying their Glow latent representations. Neither of these
		classifiers involves representation learning.
    }
    \begin{center}
    \resizebox{0.9\linewidth}{!}{\begin{tabular}{@{}lP{0.03cm}ccP{0.05cm}cP{0.025cm}cP{0.05cm}cP{0.025cm}c@{}}
    \toprule
    & & \multicolumn{2}{c}{Majority class} & & \multicolumn{3}{c}{Acc (ResNet-18)} & & \multicolumn{3}{c}{Acc ($\glowpoint$)} \\
    \cmidrule{3-4}
    \cmidrule{6-8}
    \cmidrule{10-12}
    Task & & Valid & Test && Valid && Test && Valid && Test \\
    \midrule
    \texttt{Smiling} && 51.7 & 52.6 && \textbf{92.1 $\pm$ 0.2} && \textbf{90.9 $\pm$ 0.7} && 89.4 $\pm$ 0.1 && 87.2 $\pm$ 1.1 \\
    \texttt{Earrings} && 80.9 & 78.2 && \textbf{86.2 $\pm$ 0.8} && \textbf{88.2 $\pm$ 1.1} && 84.7 $\pm$ 0.1 && 85.2 $\pm$ 0.9 \\
    \bottomrule
\end{tabular}
}
    \end{center}
    \label{tbl:appendix-base-acc}
\end{table*}

\begin{table*}[hbt!]
    \caption{
		Empirical individual fairness of the base classifiers evaluated via sampling.
		These results come with no provable guarantees and serve as upper bounds
		of the certified individual fairness.
    }
    \begin{center}
    \resizebox{0.975\linewidth}{!}{\begin{tabular}{@{}lP{0.03cm}lP{0.03cm}cP{0.025cm}cP{0.05cm}cP{0.025cm}c@{}}
    \toprule
    & & & & \multicolumn{3}{c}{Emp. Fair (ResNet-18)} & & \multicolumn{3}{c}{Emp. Fair ($\glowpoint$)} \\
    \cmidrule{5-7}
    \cmidrule{9-11}
    Task & & Sensitive attribute(s) && Valid && Test && Valid && Test \\
    \midrule
    \multirow{6}{*}{\texttt{Smiling}} && \texttt{Pale\_Skin} && 74.1 $\pm$ 1.0 && 75.2 $\pm$ 1.1 && 75.8 $\pm$ 0.6 && 79.9 $\pm$ 1.2 \\
    && \texttt{Young} && 87.7 $\pm$ 0.5 && 90.1 $\pm$ 0.7 && 95.2 $\pm$ 0.6 && 96.8 $\pm$ 0.9 \\
    && \texttt{Blond\_Hair} && 89.1 $\pm$ 1.1 && 89.4 $\pm$ 1.5 && 81.9 $\pm$ 1.6 && 84.6 $\pm$ 2.9 \\
    && \texttt{Heavy\_Makeup} && 82.3 $\pm$ 1.0 && 82.5 $\pm$ 1.2 && 74.9 $\pm$ 1.6 && 78.2 $\pm$ 2.3 \\
    && \texttt{Pale+Young} && 71.5 $\pm$ 0.9 && 72.6 $\pm$ 1.1 && 75.8 $\pm$ 0.6 && 79.9 $\pm$ 1.2 \\
    && \texttt{Pale+Young+Blond} && 70.3 $\pm$ 0.5 && 70.6 $\pm$ 0.9 && 72.5 $\pm$ 0.6 && 76.9 $\pm$ 1.1 \\
    \midrule
    \multirow{4}{*}{\texttt{Earrings}} && \texttt{Pale\_Skin} && 92.8 $\pm$ 0.9 && 90.4 $\pm$ 1.2 && 91.5 $\pm$ 1.0 && 91.5 $\pm$ 1.6 \\
    && \texttt{Young} && 90.6 $\pm$ 1.8 && 87.7 $\pm$ 1.5 && 93.0 $\pm$ 1.2 && 94.7 $\pm$ 1.0 \\
    && \texttt{Blond\_Hair} && 89.7 $\pm$ 2.2 && 86.9 $\pm$ 2.5 && 88.3 $\pm$ 1.9 && 89.7 $\pm$ 2.2 \\
    && \texttt{Heavy\_Makeup} && 85.8 $\pm$ 3.4 && 82.2 $\pm$ 3.3 && 74.4 $\pm$ 4.4 && 73.3 $\pm$ 3.8 \\
    \bottomrule
\end{tabular}
}
    \end{center}
    \label{tbl:appendix-empirical-fair}
\end{table*}

\section{Hyperparameter Tuning}
\label{app:hyperparam-tuning}

In this section, we perform an extensive hyperparameter search in order to select
suitable values for the hyperparameters. We evaluate on 311 samples from the
\emph{validation} set of CelebA (again, every 64-th), on the \texttt{Smiling} task
with sensitive attributes \texttt{Pale\_Skin} and \texttt{Young}. Afterwards,
we reuse the same hyperparameter values for all tasks with very minor changes
(which we verify by running the experiments on the validation set first).
The tunable hyperparameters, as well as the range of values that we consider
about them, are as follows:
\begin{itemize}
	\item Adversarial loss weight:\\
	$\lambda_2 \in \{0, 0.001, 0.0025, 0.005, 0.01, 0.025, 0.05, 0.1, 0.25\}$
	\item Gaussian noise added during center smoothing of $\lassi$:\\
	$\CSsigma \in \{0.5, 0.55, 0.6, 0.65, 0.7, 0.75\}$
	\item Gaussian noise added during randomized
	smoothing of $\consumer$:\\
	$\RSsigma \in \{0.1, 0.25, 0.5, 1, 2.5, 5, 10, 25\}$
\end{itemize}

\medbreak

\noindent\textbf{Tuning $\CSsigma$ and the baselines} {}
We begin with selecting the value for $\CSsigma$. It is not used during
the training of $\lassi$ and $\consumer$, but is an integral part of the
center smoothing computation which is performed during inference and is the
most time-consuming component of the model pipeline. More concretely,
both $\CScenter = \smoothlassi\left(\glowpoint\right)$ and $\CSradius$ depend on
$\CSsigma$, in turn affecting both the accuracy and the certified
individual fairness. We evaluate the Naive model with all candidate
values for $\CSsigma$ and show the results
in~\cref{tbl:appendix-naive}. We observe very little variation
in accuracy, while the best certified individual fairness
and the smallest average center smoothing radii are
obtained at $\CSsigma = 0.6$ and $0.65$. While there is no significant
difference in performace between these two configurations, we expect that the
slightly larger value for $\CSsigma$ would generally produce
smaller center smoothing radii, leading to increased certified
fairness. Therefore, we set $\CSsigma = 0.65$ for all experiments
(except for FairFace, where we use $\epsilon = 0.5$ and
scale $\CSsigma$ correspondingly, \ie $\CSsigma = 0.325$).
Using the same $\CSsigma$ values for both the baselines and LASSI allows us to
attribute the improvements of the results to the additional training
mechanisms that we apply and not merely to different
hyperparameter values.

We perform a similar evaluation on the validation set of the other
baseline, DataAug, and from the
results in~\cref{tbl:appendix-naive,tbl:appendix-data-aug} we
set $\RSsigma = 10$ for both Naive and DataAug. Although
$\RSsigma = 5$ seems to work slightly better for \texttt{Young}, we
remark that \texttt{Young} is also the fairest of all
considered sensitive attributes, so we choose a more conservative value
that would be suitable for all of them.

\medskip

\noindent\textbf{Tuning $\lambda_2$} {}
Next, we incorporate the adversarial loss weight $\lambda_2$ to the training and
explore its impact on the model in~\cref{tbl:appendix-adversarial}.
The certified individual fairness increases with increasing
$\lambda_2$, until $\lambda_2 = 0.05$, and stays at
the same level afterwards. Interestingly, the accuracy is mostly unaffected.
We set $\lambda_2 = 0.05$ and $\RSsigma = 2.5$ for LASSI, as they give most of
the fairness boost obtained from adversarial
training, while keeping the accuracy high. Notably, the hyperparameter tuning
demonstrates that LASSI successfully enforces and certifies individual fairness
for a wide range of hyperparameter values and is not highly sensitive to them.

\begin{table*}[hbt!]
    \caption{
		Results of Naive on the validation subset of CelebA
		for different values of $\CSsigma$ and $\RSsigma$.
		The third column contains the mean center smoothing radii
		corresponding to the different $\CSsigma$ values. Smaller
		is generally better for certified individual fairness
		(see the condition in~\cref{alg:certify-fairness}).
    }
    \begin{center}
        \resizebox{\linewidth}{!}{\begin{tabular}{@{}lcccccP{0.8cm}P{0.8cm}P{0.8cm}P{0.8cm}P{0.8cm}P{0.8cm}P{0.8cm}P{0.8cm}@{}}
    \toprule
    & & & & & & \multicolumn{8}{c}{$\RSsigma$} \\
    \cmidrule{7-14}
    Sens. attribute && $\CSsigma$ && Mean $\CSradius$ & Metric & $0.1$ & $0.25$ & $0.5$ & $1$ & $2.5$ & $5$ & $10$ & $25$ \\
    \midrule
    \multirow{14}{*}{\texttt{Pale\_Skin}} && \multirow{2}{*}{0.5} && \multirow{2}{*}{42.25} & Acc & 87.8 & 87.8 & 87.5 & 88.4 & 89.1 & 88.7 & 88.7 & 84.6 \\
    &&&& & Fair & 0 & 0 & 0 & 0 & 0 & 0 & 0 & 0 \\
    \\[-1.25ex]
    && \multirow{2}{*}{0.55} && \multirow{2}{*}{34.19} & Acc & 87.8 & 87.8 & 87.8 & 88.4 & 88.7 & 88.7 & 88.4 & 84.6 \\
    &&&& & Fair & 0 & 0 & 0 & 0 & 0 & 0 & 0 & 0 \\
    \\[-1.25ex]
    && \multirow{2}{*}{0.6} && \multirow{2}{*}{\textbf{33.34}} & Acc & 87.8 & 87.5 & 87.8 & 88.4 & 88.7 & 88.7 & 88.7 & 84.6 \\
    &&&& & Fair & 0 & 0 & 0 & 0 & 0 & 0 & \textbf{1.0} & 0 \\
    \\[-1.25ex]
    && \multirow{2}{*}{0.65} && \multirow{2}{*}{33.37} & Acc & 87.5 & 87.5 & 87.8 & 88.4 & 88.8 & 88.7 & 88.4 & 84.6 \\
    &&&& & Fair & 0 & 0 & 0 & 0 & 0 & 0 & \textbf{1.0} & 0 \\
    \\[-1.25ex]
    && \multirow{2}{*}{0.7} && \multirow{2}{*}{33.72} & Acc & 87.5 & 87.5 & 87.5 & 88.4 & 88.4 & 88.7 & 88.1 & 84.6 \\
    &&&& & Fair & 0 & 0 & 0 & 0 & 0 & 0 & \textbf{1.0} & 0 \\
    \\[-1.25ex]
    && \multirow{2}{*}{0.75} && \multirow{2}{*}{34.18} & Acc & 87.8 & 88.1 & 88.1 & 88.4 & 88.7 & 89.1 & 88.1 & 84.6 \\
    &&&& & Fair & 0 & 0 & 0 & 0 & 0 & 0 & \textbf{1.0} & 0 \\
    \midrule
    \multirow{4}{*}{\texttt{Young}} && \multirow{2}{*}{0.6} && \multirow{2}{*}{\textbf{8.16}} & Acc & 88.1 & 88.1 & 87.8 & 87.8 & 88.7 & 88.7 & 88.1 & 85.2 \\
    &&&& & Fair & 0 & 0 & 0 & 5.1 & 36.3 & \textbf{58.8} & 58.5 & 39.9 \\
    \\[-1.25ex]
    && \multirow{2}{*}{0.65} && \multirow{2}{*}{\textbf{8.16}} & Acc & 88.1 & 88.1 & 87.8 & 87.8 & 88.7 & 88.7 & 88.1 & 84.9 \\
    &&&& & Fair & 0 & 0 & 0 & 4.8 & 36.3 & \textbf{58.8} & 58.2 & 39.5 \\
    \bottomrule
\end{tabular}
}
    \end{center}
	\label{tbl:appendix-naive}
\end{table*}

\begin{table*}[hbt!]
    \caption{
		Results of the DataAug baseline on the validation set
		of CelebA for $\CSsigma = 0.65$ and different values of
		$\RSsigma$.
    }
    \begin{center}
        \resizebox{\linewidth}{!}{\begin{tabular}{@{}lcccccP{0.8cm}P{0.8cm}P{0.8cm}P{0.8cm}P{0.8cm}P{0.8cm}P{0.8cm}P{0.8cm}@{}}
    \toprule
    & & & & & & \multicolumn{8}{c}{$\RSsigma$} \\
    \cmidrule{7-14}
    Sens. attribute && $\CSsigma$ && Mean $\CSradius$ & Metric & $0.1$ & $0.25$ & $0.5$ & $1$ & $2.5$ & $5$ & $10$ & $25$ \\
    \midrule
    \multirow{2}{*}{\texttt{Pale\_Skin}} && \multirow{2}{*}{0.65} && \multirow{2}{*}{14.52} & Acc & 87.5 & 87.5 & 87.8 & 87.8 & 88.7 & 89.4 & 88.7 & 84.9 \\
    &&&& & Fair & 0 & 0 & 0 & 0 & 0 & 28.3 & \textbf{31.5} & 10.0 \\
    \multirow{2}{*}{\texttt{Young}} && \multirow{2}{*}{0.65} && \multirow{2}{*}{7.09} & Acc & 87.5 & 87.8 & 87.8 & 89.1 & 88.7 & 88.7 & 88.7 & 84.9 \\
    &&&& & Fair & 0 & 0 & 0 & 1.6 & 46.6 & \textbf{65.6} & 65.0 & 48.9 \\
    \bottomrule
\end{tabular}
}
    \end{center}
	\label{tbl:appendix-data-aug}
\end{table*}

\begin{table*}[hbt!]
    \caption{
		Results of LASSI on the validation subset of CelebA
		for different values of $\lambda_2$ and $\RSsigma$, while
		keeping $\CSsigma = 0.65$. The certified individual fairness
		increases with increasing $\lambda_2$, until
		the $\lambda_2 = 0.05$ level.
    }
    \begin{center}
        \resizebox{0.975\linewidth}{!}{\begin{tabular}{@{}lccP{0.8cm}P{0.8cm}P{0.8cm}P{0.8cm}P{0.8cm}P{0.8cm}P{0.8cm}P{0.8cm}@{}}
    \toprule
    & & & \multicolumn{8}{c}{$\RSsigma$} \\
    \cmidrule{4-11}
    Sens. attribute & $\lambda_2$ & Metric & $0.1$ & $0.25$ & $0.5$ & $1$ & $2.5$ & $5$ & $10$ & $25$ \\
    \midrule
    \multirow{20}{*}{\texttt{Pale\_Skin}} & \multirow{2}{*}{$0.001$} & Acc & 86.5 & 86.8 & 87.1 & 87.5 & 89.1 & 89.4 & 88.1 & 84.9 \\
    & & Fair & 0 & 0 & 0 & 0 & 0 & \textbf{13.2} & 12.9 & 1.9 \\
    \\[-1ex]
    &\multirow{2}{*}{$0.0025$} & Acc & 87.8 & 88.1 & 88.4 & 88.7 & 90.4 & 89.1 & 86.5 & 83.3 \\
    && Fair & 0 & 0 & 0 & 0 & 15.4 & \textbf{27.3} & 24.8 & 5.5 \\
    \\[-1ex]
    &\multirow{2}{*}{$0.005$} & Acc & 87.8 & 87.8 & 88.1 & 87.8 & 89.7 & 89.1 & 87.5 & 82.0 \\
    && Fair & 0 & 0 & 0 & 0 & 35.0 & \textbf{40.5} & 37.0 & 15.4 \\
    \\[-1ex]
    &\multirow{2}{*}{$0.01$} & Acc & 88.1 & 88.1 & 87.8 & 88.1 & 89.4 & 90.0 & 87.5 & 82.0 \\
    && Fair & 0 & 0 & 0 & 9.3 & 46.0 & \textbf{49.5} & 47.6 & 27.7 \\
    \\[-1ex]
    &\multirow{2}{*}{$0.025$} & Acc & 88.4 & 88.1 & 88.1 & 88.4 & 89.1 & 89.7 & 87.5 & 82.3 \\
    && Fair & 0 & 1.9 & 9.6 & 49.2 & 64.3 & \textbf{66.2} & 64.0 & 47.9 \\
    \\[-1ex]
    &\multirow{2}{*}{$0.05$} & Acc & 87.8 & 87.8 & 88.1 & 88.1 & 89.7 & 89.4 & 86.8 & 83.0 \\
    && Fair & 45.0 & 97.1 & 97.7 & \textbf{98.1} & 96.1 & 96.1 & 95.5 & 93.6 \\
    \\[-1ex]
    &\multirow{2}{*}{$0.1$} & Acc & 86.5 & 86.5 & 86.5 & 86.8 & 86.5 & 85.9 & 83.3 & 76.8 \\
    && Fair & 57.9 & 93.6 & 93.9 & 94.5 & \textbf{96.8} & 96.1 & 94.9 & 88.1 \\
    \\[-1ex]
    &\multirow{2}{*}{$0.25$} & Acc & 87.1 & 87.1 & 87.1 & 87.5 & 87.5 & 85.9 & 79.4 & 67.8 \\
    && Fair & 96.8 & 96.1 & 96.1 & 96.1 & \textbf{98.1} & 97.4 & 93.2 & 79.7 \\
    \midrule
    \multirow{7}{*}{\texttt{Young}} & \multirow{2}{*}{$0.05$} & Acc & 89.1 & 89.1 & 88.1 & 89.4 & 89.4 & 89.1 & 88.7 & 84.6 \\
    & & Fair & 97.1 & 97.7 & 98.4 & \textbf{99.0} & \textbf{99.0} & 98.7 & 96.5 & 96.1 \\
    \\[-1ex]
    &\multirow{2}{*}{$0.1$} & Acc & 88.1 & 88.7 & 89.4 & 89.4 & 88.7 & 88.7 & 87.8 & 82.6 \\
    & & Fair & 58.5 & 94.9 & 94.9 & 96.8 & \textbf{97.1} & 95.8 & 96.1 & 92.3 \\
    \\[-1ex]
    &\multirow{2}{*}{$0.25$} & Acc & 88.4 & 88.4 & 88.4 & 88.7 & 88.4 & 88.1 & 86.8 & 77.8 \\
    & & Fair & 98.4 & 98.1 & 98.4 & \textbf{99.4} & \textbf{99.4} & 98.7 & 95.2 & 89.4 \\
    \bottomrule
\end{tabular}
}
    \end{center}
	\label{tbl:appendix-adversarial}
\end{table*}

\FloatBarrier

\subsubsection{Selected experiment hyperparameters}

Here, we summarize the hyperparameter values selected for the final experiments.
We use $\epsilon = 1$ for all similarity set definitions except the experiments with:
(i) the alternative attribute vectors
from~\cite{denton2019detecting,li2021discover}, where $\epsilon = 10$,
and (ii) FairFace, where $\epsilon = 0.5$. We maintain the $\epsilon/\CSsigma$ ratio,
which impacts center smoothing, setting
$\CSsigma = 0.65$ by default (as stated in the sections above) and using
$\CSsigma = 6.5$ and $0.325$ when $\epsilon = 10$ and $0.5$ respectively.
Our smoothing arguments are consistent with prior
work~\cite{cohen2019certified,kumar2021center}:
\begin{itemize}[noitemsep,nolistsep]
	\item Randomized smoothing~\cite{cohen2019certified}: $\RSprob = 0.001$,
	$N_{rs} = \num[group-separator={,}]{100000}$,
	$N_{0,rs} = \num[group-separator={,}]{2000}$.
	\item Center smoothing~\cite{kumar2021center}: $\CSprob = 0.01$,
	$N_{cs} = \num[group-separator={,}]{10000}$,
	$N_{0,cs} = \num[group-separator={,}]{10000}$.
\end{itemize}
The rest of the model hyperparameters are listed
in~\cref{tbl:appendix-selected-hyperparams}.
In the CelebA 64$\times$64 and 128$\times$128 setups, we run LASSI with
$\lambda_2 = 0.25$ for the (target=\texttt{Earrings}, sensitive=\texttt{Makeup})
pair of attributes because of the high correlation between them. We train
the representation $\lassi$ for 20 epochs in the transfer experiments
(CelebA, FairFace) and 5 epochs otherwise. The linear classifier $\consumer$ is
trained for 1 epoch. We generally set a lower value to $\RSsigma$ when the task
is more difficult and the downstream classifier is therefore less confident.
Overall, we remark that the hyperparameter values
are similar and within the same range for all models and experiments, meaning that
our approach does not require substantial fine-tuning.

\begin{table*}[hbt!]
	\caption{
		Hyperparameters used for the different model and experiment setups.
	}
	\begin{center}
		\resizebox{\linewidth}{!}{\begin{tabular}{@{}lll@{}}
    \toprule
    & Model / &  \\
    Dataset & Experiment(s) & Hyperparameters \\
    \midrule
    \multirow{4}{*}{CelebA} & 64$\times$64 and & $\lambda_1 = 1$;
    $\lambda_2 = 0$ (Naive, DataAug) and $0.05$ (LASSI);
    $\lambda_3 = 0$; \\
    & 128$\times$128 &
    $\RSsigma = 10$ (Naive, DataAug) and $2.5$ (LASSI); $s = 10$ (DataAug, LASSI). \\
    & & \\
    & Transfer & $\lambda_1 = 0$; $\lambda_2 = 0.05$; $\lambda_3 = 0.1$;
    $\RSsigma = 0.5$; $s = 10$.\\
    \midrule
    \multirow{3}{*}{FairFace} & Naive & $\lambda_1 = 1$; $\lambda_2 = \lambda_3 = 0$;
    $\RSsigma = 5$ (\texttt{Age-2}) and $0.1$
    (\texttt{Age-3}, \texttt{Age} (all)). \\
    & LASSI & $\lambda_1 = 1$; $\lambda_2 = 0.1$; $\lambda_3 = 0$;
    $\RSsigma = 0.25$; $s = 10$.\\
    & Transfer & $\lambda_1 \in \{0, 0.001, 0.01\}$; $\lambda_2 = \lambda_3 = 0.1$;
    $\RSsigma = 0.1$; $s = 10$. \\
    \midrule
    3D Shapes & Naive & $\lambda_1 = 1$; $\lambda_2 = \lambda_3 = 0$;
    $\RSsigma = 5$. \\
    (\cref{app:3dshapes})& LASSI & $\lambda_1 = 1$; $\lambda_2 = 0.1$; $\lambda_3 = 0$;
    $\RSsigma = 1$; $s = 100$. \\
    \bottomrule
\end{tabular}
}
	\end{center}
	\label{tbl:appendix-selected-hyperparams}
\end{table*}

\section{More Experimental Results on CelebA}
\label{app:more-experimental-results-stddev}

This section provides further details about the experiments on the CelebA dataset
with the standard attribute vector from~\cite{kingma2018glow},
$\attrvec = \generativepoint_{G, pos} - \generativepoint_{G, neg}$
(\cref{sec:similarity}).

\subsubsection{64$\times$64 images}

\cref{tbl:appendix-table1-acc-fair-stddev}
contains the means and the standard deviations of the accuracies and the certified
individual fairness of the CelebA 64$\times$64 experiments summarized
in~\cref{tbl:results}, averaged over 5 runs. The standard deviation of Naive
and DataAug's fairness is high, while LASSI consistently enforces certified
individual fairness with low variance.

\begin{table*}[hbt!]
	\caption{
		Means and standard deviations of the accuracies and
		the certified individual fairness reported
		in~\cref{tbl:results}, averaged over $5$ runs with different random seeds
		on the \texttt{Smiling} (rows 1-6) and \texttt{Earrings} (rows 7-10) tasks.
	}
	\begin{center}
		\resizebox{\linewidth}{!}{\begin{tabular}{@{}lcccccccccccccc@{}}
	\toprule
	&& \multicolumn{3}{c}{Naive} &&& \multicolumn{3}{c}{DataAug} &&& \multicolumn{3}{c}{LASSI (ours)} \\
	\cmidrule{3-5}
	\cmidrule{7-10}
	\cmidrule{13-15}
	Sens. attribs.: && Acc && Fair &&& Acc && Fair &&& Acc && Fair \\
	\midrule
	\texttt{Pale\_Skin}     && \textbf{86.3 $\pm$ 1.5} && 0.6 $\pm$ 0.5 &&& 85.7 $\pm$ 1.2 && 12.2 $\pm$ 14.7 &&& 85.9 $\pm$ 1.3 && \textbf{98.0 $\pm$ 0.5} \\
	\texttt{Young}              && \textbf{86.3 $\pm$ 1.8} && 38.2 $\pm$ 23.4 &&& 85.9 $\pm$ 1.6 && 43.0 $\pm$ 30.7 &&& \textbf{86.3 $\pm$ 1.3} && \textbf{98.8 $\pm$ 0.6} \\
	\texttt{Blond\_Hair}        && 86.3 $\pm$ 1.6 && 3.4 $\pm$ 3.1 &&& \textbf{86.6 $\pm$ 1.0} && 9.4 $\pm$ 10.0 &&& 86.4 $\pm$ 1.0 && \textbf{94.7 $\pm$ 1.5} \\
	\texttt{Heavy\_Makeup}      && \textbf{86.3 $\pm$ 1.1} && 0.4 $\pm$ 0.4 &&& 85.3 $\pm$ 1.7 && 13.7 $\pm$ 8.8 &&& 85.6 $\pm$ 1.6 && \textbf{91.3 $\pm$ 8.1} \\
	\texttt{P+Y}         && \textbf{86.0 $\pm$ 1.5} && 0.4 $\pm$ 0.4 &&& 85.8 $\pm$ 1.4 && 9.9 $\pm$ 12.7 &&& 85.8 $\pm$ 0.9 && \textbf{97.3 $\pm$ 0.9} \\
	\texttt{P+Y+B}   && 86.2 $\pm$ 1.7 && 0.0 $\pm$ 0.0 &&& \textbf{86.4 $\pm$ 1.0} && 3.6 $\pm$ 3.8 &&& 85.5 $\pm$ 0.4 && \textbf{86.5 $\pm$ 2.7} \\
	\midrule
	\texttt{Pale\_Skin}     && 81.3 $\pm$ 2.2 && 24.3 $\pm$ 35.6 &&& 81.0 $\pm$ 2.3 && 40.4 $\pm$ 32.6 &&& \textbf{85.0 $\pm$ 0.5} && \textbf{98.5 $\pm$ 0.9} \\
	\texttt{Young}              && 81.4 $\pm$ 2.2 && 59.2 $\pm$ 18.0 &&& 79.9 $\pm$ 1.4 && 72.0 $\pm$ 24.1 &&& \textbf{84.5 $\pm$ 1.0} && \textbf{98.0 $\pm$ 1.1} \\
	\texttt{Blond\_Hair}        && 81.4 $\pm$ 2.2 && 9.2 $\pm$ 17.5 &&& 82.2 $\pm$ 2.8 && 30.5 $\pm$ 40.9 &&& \textbf{84.8 $\pm$ 0.5} && \textbf{96.2 $\pm$ 2.6} \\
	\texttt{Heavy\_Makeup}      && 81.6 $\pm$ 1.9 && 20.5 $\pm$ 13.0 &&& 80.3 $\pm$ 1.9 && 49.2 $\pm$ 37.0 &&& \textbf{82.3 $\pm$ 0.6} && \textbf{98.7 $\pm$ 0.7} \\
	\bottomrule
\end{tabular}
}
	\end{center}
	\label{tbl:appendix-table1-acc-fair-stddev}
\end{table*}

\begin{table*}[hbt!]
	\caption{
		Empirical evaluation of the individual fairness of the models computed by
		comparing their predictions on the original test samples to
		the model predictions on the endpoints of the corresponding similarity sets.
	}
	\begin{center}
		\resizebox{0.925\linewidth}{!}{\begin{tabular}{@{}lP{0.03cm}lcP{2cm}cP{2cm}cP{2cm}@{}}
    \toprule
    Task & & Sensitive attribute(s) && Naive && DataAug && LASSI (ours) \\
    \midrule
    \multirow{6}{*}{\texttt{Smiling}} && \texttt{Pale\_Skin} & & 78.4 $\pm$ 2.1 && 90.1 $\pm$ 1.9 && \textbf{99.6 $\pm$ 0.2} \\
    & & \texttt{Young} & & 95.3 $\pm$ 0.4 && 96.7 $\pm$ 0.5 && \textbf{99.6 $\pm$ 0.2} \\
    & & \texttt{Blond\_Hair} & & 83.3 $\pm$ 0.7 && 93.9 $\pm$ 1.5 && \textbf{99.2 $\pm$ 0.4} \\
    & & \texttt{Heavy\_Makeup} & & 75.8 $\pm$ 2.4 && 88.3 $\pm$ 0.8 && \textbf{97.9 $\pm$ 1.6} \\
    & & \texttt{Pale+Young} & & 78.0 $\pm$ 2.0 && 89.0 $\pm$ 2.2 && \textbf{99.4 $\pm$ 0.5} \\
    & & \texttt{Pale+Young+Blond} & & 77.9 $\pm$ 2.1 && 87.4 $\pm$ 0.9 && \textbf{96.9 $\pm$ 0.7} \\
    \midrule
    \multirow{4}{*}{\texttt{Earrings}} && \texttt{Pale\_Skin} & & 97.1 $\pm$ 1.6 && 99.1 $\pm$ 0.7 && \textbf{99.5 $\pm$ 0.4} \\
    & & \texttt{Young} & & 98.5 $\pm$ 1.4 && \textbf{99.5 $\pm$ 0.5} && 99.2 $\pm$ 0.4 \\
    & & \texttt{Blond\_Hair} & & 96.7 $\pm$ 3.4 && 98.5 $\pm$ 0.4 && \textbf{99.1 $\pm$ 0.7} \\
    & & \texttt{Heavy\_Makeup} & & 92.2 $\pm$ 6.6 && 98.1 $\pm$ 1.1 && \textbf{99.7 $\pm$ 0.3} \\
    \bottomrule
\end{tabular}
}
	\end{center}
	\label{tbl:appendix-endpoints-fair}
\end{table*}

Moreover, in~\cref{tbl:appendix-endpoints-fair}
we check for what fraction of the test subset the models classify the similarity set
endpoints the same as the original data point. Note that this is again another
empirical estimate, serving as an upper bound of the certified individual fairness of
the models. Nevertheless, LASSI outperforms the baselines on that metric as well.
More importantly, out of all 150 combinations of models, tasks and sensitive attributes
(3 model types, 10 task-attribute pairs, 5 random seeds), in 8 combinations
there is only 1 test sample which we certify as individually fair but the endpoints
classifications mismatch. In all other combinations, no such situation
occurs, serving as another test for the correctness of our certificates. One test sample
out of 312 is 0.32\%, which is within our confidence of $1 - \CSprob - \RSprob = 98.9\%$.

\subsubsection{128$\times$128 images}

Keeping all hyperparameters the same, we evaluate LASSI on images of size
128$\times$128. The results in~\cref{tbl:appendix-128-acc-fair-stddev} indicate that
LASSI increases the certified individual fairness in this setting as well, while also
slightly improving the classification accuracy. We attribute this to the richer
and larger latent space of Glow, which is potentially more easily separable in this
case.

\begin{table*}[hbt!]
	\caption{
		Evaluation of LASSI on 128$\times$128-dimensional images, demonstrating that it
		significantly increases the certified individual fairness for larger images
		as well. Evaluated tasks:
		\texttt{Smiling} (rows 1-6) and \texttt{Earrings} (rows 7-10).
	}
	\begin{center}
		\resizebox{\linewidth}{!}{\begin{tabular}{@{}lcccccccccccccc@{}}
	\toprule
	&& \multicolumn{3}{c}{Naive} &&& \multicolumn{3}{c}{DataAug} &&& \multicolumn{3}{c}{LASSI (ours)} \\
	\cmidrule{3-5}
	\cmidrule{7-10}
	\cmidrule{13-15}
	Sens. attribs.: && Acc && Fair &&& Acc && Fair &&& Acc && Fair \\
	\midrule
	\texttt{Pale\_Skin}     && 88.8 $\pm$ 1.0 && 0.0 $\pm$ 0.0 &&& 89.6 $\pm$ 0.5 && 0.0 $\pm$ 0.0 &&& \textbf{90.0 $\pm$ 1.1} && \textbf{70.6 $\pm$ 14.2} \\
	\texttt{Young}              && 88.7 $\pm$ 0.7 && 46.0 $\pm$ 16.2 &&& 88.8 $\pm$ 1.0 && 47.6 $\pm$ 20.2 &&& \textbf{89.7 $\pm$ 0.7} && \textbf{97.2 $\pm$ 1.6} \\
	\texttt{Blond\_Hair}        && 88.8 $\pm$ 0.9 && 0.1 $\pm$ 0.1 &&& 89.4 $\pm$ 1.1 && 0.0 $\pm$ 0.0 &&& \textbf{90.1 $\pm$ 0.8} && \textbf{77.8 $\pm$ 10.2} \\
	\texttt{Heavy\_Makeup}      && 89.0 $\pm$ 0.9 && 2.5 $\pm$ 3.5 &&& 89.6 $\pm$ 1.1 && 30.4 $\pm$ 20.7 &&& \textbf{90.2 $\pm$ 0.3} && \textbf{87.6 $\pm$ 3.9} \\
	\texttt{P+Y}         && 88.8 $\pm$ 1.0 && 0.0 $\pm$ 0.0 &&& 89.4 $\pm$ 1.3 && 8.7 $\pm$ 16.5 &&& \textbf{90.2 $\pm$ 0.5} && \textbf{69.4 $\pm$ 9.7} \\
	\texttt{P+Y+B}   && 88.7 $\pm$ 0.8 && 0.0 $\pm$ 0.0 &&& 89.9 $\pm$ 1.5 && 4.4 $\pm$ 9.6 &&& \textbf{90.2 $\pm$ 0.7} && \textbf{72.7 $\pm$ 5.0} \\
	\midrule
	\texttt{Pale\_Skin}     && 80.1 $\pm$ 1.4 && 0.0 $\pm$ 0.0 &&& 80.1 $\pm$ 2.5 && 0.1 $\pm$ 0.1 &&& \textbf{84.4 $\pm$ 0.9} && \textbf{90.4 $\pm$ 2.5} \\
	\texttt{Young}              && 80.2 $\pm$ 1.4 && 73.5 $\pm$ 20.4 &&& 80.3 $\pm$ 1.5 && 78.2 $\pm$ 18.1 &&& \textbf{85.5 $\pm$ 1.4} && \textbf{96.4 $\pm$ 1.7} \\
	\texttt{Blond\_Hair}        && 80.2 $\pm$ 1.4 && 0.0 $\pm$ 0.0 &&& 80.6 $\pm$ 2.0 && 0.0 $\pm$ 0.0 &&& \textbf{83.9 $\pm$ 0.9} && \textbf{89.7 $\pm$ 4.0} \\
	\texttt{Heavy\_Makeup}      && 80.3 $\pm$ 1.4 && 42.1 $\pm$ 15.9 &&& 80.1 $\pm$ 1.9 && 65.1 $\pm$ 31.1 &&& \textbf{81.7 $\pm$ 1.3} && \textbf{98.3 $\pm$ 1.3} \\
	\bottomrule
\end{tabular}
}
	\end{center}
	\label{tbl:appendix-128-acc-fair-stddev}
\end{table*}

\subsubsection{Transfer learning}

\cref{tbl:appendix-base-acc-transfer} contains the base standard accuracies on the
transfer tasks. \cref{tbl:appendix-transfer-stddevs} reports the means and
the standard deviations of LASSI on the \texttt{Smiling} task when solved
in a transfer learning setting.

\begin{table*}[hbt!]
	\caption{
		Baseline accuracies on the transfer CelebA
		tasks. As before, the ResNet-18 classifier takes as an input
		the original images, while the $\glowpoint$ classifier is a fully
		connected network classifying their Glow latent representations.
		Neither of these classifiers involves representation learning.
	}
	\begin{center}
		\resizebox{0.9\linewidth}{!}{\begin{tabular}{@{}lP{0.03cm}ccP{0.05cm}cP{0.025cm}cP{0.05cm}cP{0.025cm}c@{}}
    \toprule
    & & \multicolumn{2}{c}{Majority class} & & \multicolumn{3}{c}{Acc (ResNet-18)} & & \multicolumn{3}{c}{Acc ($\glowpoint$)} \\
    \cmidrule{3-4}
    \cmidrule{6-8}
    \cmidrule{10-12}
    Task & & Valid & Test && Valid && Test && Valid && Test \\
    \midrule
    \texttt{Smiling} && 51.7 & 52.6 && \textbf{92.1 $\pm$ 0.2} && \textbf{90.9 $\pm$ 0.7} && 89.4 $\pm$ 0.1 && 87.2 $\pm$ 1.1 \\
    \texttt{High\_Cheeks} && 55.1 & 51.9 && \textbf{87.2 $\pm$ 0.2} && \textbf{86.8 $\pm$ 0.4} && 84.3 $\pm$ 0.1 && 83.8 $\pm$ 0.7 \\
    \texttt{Mouth\_Open} && 51.8 & 53.8 && \textbf{92.7 $\pm$ 0.3} && \textbf{92.9 $\pm$ 0.7} && 88.1 $\pm$ 0.2 && 89.6 $\pm$ 1.1 \\
    \texttt{Lipstick} && 55.4 & 54.8 && \textbf{91.5 $\pm$ 0.2} && 90.5 $\pm$ 0.8 && 89.2 $\pm$ 0.1 && \textbf{90.6 $\pm$ 1.1} \\
    \texttt{Heavy\_Makeup} && 61.0 & 58.7 && \textbf{90.2 $\pm$ 0.4} && \textbf{89.9 $\pm$ 0.4} && 87.8 $\pm$ 0.1 && 88.6 $\pm$ 1.1 \\
    \texttt{Wavy\_Hair} && 72.3 & 65.1 && \textbf{82.7 $\pm$ 1.8} && 76.3 $\pm$ 3.3 && 80.9 $\pm$ 0.5 && \textbf{81.7 $\pm$ 0.4} \\
    \texttt{Eyebrows} && 74.2 & 71.8 && \textbf{83.5 $\pm$ 0.5} && \textbf{81.1 $\pm$ 0.6} && 80.1 $\pm$ 0.1 && 79.4 $\pm$ 1.6 \\
    \bottomrule
\end{tabular}
}
	\end{center}
	\label{tbl:appendix-base-acc-transfer}
\end{table*}

\begin{table*}[hbt!]
	\caption{
		Mean and standard deviation of the accuracies and the certified individual
		fairness of LASSI on \texttt{Smiling} in a transfer
		learning setting (\cref{tbl:transfer}).
	}
	\begin{center}
		\resizebox{0.75\linewidth}{!}{\begin{tabular}{@{}lP{0.03cm}lcP{2cm}cP{2cm}@{}}
    \toprule
    Task & & Sensitive attribute(s) && Acc && Fair \\
    \midrule
    \multirow{5}{*}{\texttt{Smiling}} && \texttt{Pale\_Skin} & & 86.2 $\pm$ 1.1 && 93.1 $\pm$ 2.4 \\
    & & \texttt{Young} & & 86.0 $\pm$ 1.2 && 95.4 $\pm$ 1.0 \\
    & & \texttt{Blond\_Hair} & & 85.1 $\pm$ 1.6 && 93.8 $\pm$ 1.8 \\
    & & \texttt{Pale+Young} & & 85.9 $\pm$ 0.3 && 92.2 $\pm$ 0.7 \\
    & & \texttt{Pale+Young+Blond} & & 85.1 $\pm$ 0.7 && 87.0 $\pm$ 2.3 \\
    \bottomrule
\end{tabular}
}
	\end{center}
	\label{tbl:appendix-transfer-stddevs}
\end{table*}

\FloatBarrier

\section{Different Attribute Vector Types}
\label{app:more-attrib-vectors}

In this section, we demonstrate that LASSI is independent of the actual
computation of the attribute vector $\attrvec$ and that it can improve
the individual fairness for various attribute vector types.

\subsubsection{Denton et al.~\cite{denton2019detecting}}

First, in~\cref{tbl:appendix-denton-acc-fair-stddev} we report the means and
the standard deviations of the accuracies and the certified individual fairness
from~\cref{tbl:denton-attrib-vec}. The attribute vector $\attrvec$ used here is
orthogonal to the decision boundary of the linear classifier
$\sign(\attrvec^{\top}\glowpoint + b)$~\cite{denton2019detecting}
(\cref{sec:similarity}), with its length set to $\epsilon = 10$.

\begin{table*}[hbt!]
	\caption{
		Means and standard deviations of the accuracies and
		the certified individual fairness reported
		in~\cref{tbl:denton-attrib-vec},
		averaged over $5$ runs with different random seeds
		on the \texttt{Smiling} task.
	}
	\begin{center}
		\resizebox{\linewidth}{!}{\begin{tabular}{@{}lcccccccccccccc@{}}
	\toprule
	&& \multicolumn{3}{c}{Naive} &&& \multicolumn{3}{c}{DataAug} &&& \multicolumn{3}{c}{LASSI (ours)} \\
	\cmidrule{3-5}
	\cmidrule{7-10}
	\cmidrule{13-15}
	Sens. attribs.: && Acc && Fair &&& Acc && Fair &&& Acc && Fair \\
	\midrule
	\texttt{Pale\_Skin}     && 86.4 $\pm$ 1.7 && 34.0 $\pm$ 5.4 &&& 85.9 $\pm$ 1.5 && 90.3 $\pm$ 3.9 &&& \textbf{86.5 $\pm$ 1.3} && \textbf{98.8 $\pm$ 1.2} \\
	\texttt{Young}              && 86.3 $\pm$ 1.8 && 73.1 $\pm$ 3.5 &&& 86.2 $\pm$ 1.5 && 90.3 $\pm$ 3.3 &&& \textbf{86.8 $\pm$ 1.0} && \textbf{97.9 $\pm$ 1.2} \\
	\texttt{Blond\_Hair}        && 86.2 $\pm$ 1.8 && 71.4 $\pm$ 4.0 &&& 86.1 $\pm$ 1.8 && 88.8 $\pm$ 2.7 &&& \textbf{86.7 $\pm$ 1.4} && \textbf{98.8 $\pm$ 0.7} \\
	\texttt{Heavy\_Makeup}      && 86.2 $\pm$ 1.6 && 11.5 $\pm$ 2.5 &&& 86.3 $\pm$ 1.1 && 87.4 $\pm$ 1.6 &&& \textbf{86.8 $\pm$ 1.0} && \textbf{98.8 $\pm$ 0.9} \\
	\texttt{P+Y}         && 86.2 $\pm$ 1.8 && 28.6 $\pm$ 3.4 &&& 85.8 $\pm$ 1.5 && 84.7 $\pm$ 4.1 &&& \textbf{86.5 $\pm$ 1.2} && \textbf{98.6 $\pm$ 1.8} \\
	\texttt{P+Y+B}   && 86.2 $\pm$ 1.7 && 23.7 $\pm$ 2.1 &&& 85.9 $\pm$ 1.8 && 82.2 $\pm$ 5.2 &&& \textbf{86.4 $\pm$ 1.1} && \textbf{98.7 $\pm$ 0.5} \\
	\bottomrule
\end{tabular}
}
	\end{center}
	\label{tbl:appendix-denton-acc-fair-stddev}
\end{table*}

\subsubsection{Ramaswamy et al.~\cite{ramaswamy2021fair}}

Next, we adapt the attribute vector computation proposed by~\cite{ramaswamy2021fair}
by computing sample-specific vectors
$\attrvec_i = \generativepoint_{G, i} - \generativepoint'_{G, i}$ for every
$\inputpoint_i$ from the training set, where
$\generativepoint_{G, i} = \glowenc(\inputpoint_i)$ and $\generativepoint'_{G, i}$
is as defined in~\cite[Eq. (3)]{ramaswamy2021fair}. All sample-specific
$\attrvec_i$'s share the same direction, so we can average them to obtain the global
attribute vector $\attrvec = \frac{1}{N}\sum_{i=1}^N \attrvec_i$ and set
$\epsilon = 1$.

\subsubsection{Li and Xu~\cite{li2021discover}}

Finally,~\cite{li2021discover} discover biased attributes of pre-trained classifiers.
To that end, we train a ResNet-18 on the \texttt{Smiling} task. Then, we
run~\cite{li2021discover}'s optimization procedure to iteratively find 3 biased
attribute vectors (each orthogonal to the target and to the other attribute vectors)
for that model using Glow as the generative model. We use $\epsilon = 10$ for these
vectors.

\bigskip

\noindent \cref{tab:appendix-other-attribute-vecs} shows that LASSI significantly
improves the certified individual fairness while maintaining the same high
accuracy level for~\cite{ramaswamy2021fair} and~\cite{li2021discover}, as with the
attribute vectors from~\cite{denton2019detecting,kingma2018glow}, when evaluated on
the \texttt{Smiling} task.

\begin{table*}[hbt!]
	\caption{
		Evaluation of LASSI on CelebA using sensitive attribute vectors
		from~\cite{li2021discover,ramaswamy2021fair}.
		We denote~\cite{li2021discover}'s vectors as $\attrvec_0$, $\attrvec_1$, and
		$\attrvec_2$ since they are not necessarily associated with a
		sensitive attribute
		(unlike~\cite{denton2019detecting,kingma2018glow,ramaswamy2021fair}).
		As for the vectors from~\cite{denton2019detecting,kingma2018glow}
		(\cref{tbl:results,tbl:denton-attrib-vec}), LASSI significantly
		increases certified fairness without affecting the accuracy.
	}
	\label{tab:appendix-other-attribute-vecs}
	\centering
	\resizebox{\textwidth}{!}{\begin{tabular}{@{}llcccccccccccccc@{}}
    \toprule
    & && \multicolumn{3}{c}{Naive} &&& \multicolumn{3}{c}{DataAug} &&& \multicolumn{3}{c}{LASSI (ours)} \\
    \cmidrule{4-6}
    \cmidrule{8-11}
    \cmidrule{14-16}
    $\attrvec$ & Sens. attribs.: && Acc && Fair &&& Acc && Fair &&& Acc && Fair \\
    \midrule
    \multirow{6}{*}{\cite{ramaswamy2021fair}} & \texttt{Pale\_Skin}     && 86.3 $\pm$ 1.8 && 89.0 $\pm$ 3.9 &&& 86.0 $\pm$ 1.5 && 92.4 $\pm$ 2.6 &&& \textbf{86.8 $\pm$ 1.2} && \textbf{98.6 $\pm$ 1.0} \\
    & \texttt{Young}              && 86.3 $\pm$ 1.8 && 95.1 $\pm$ 1.5 &&& 86.2 $\pm$ 1.6 && 95.6 $\pm$ 1.8 &&& \textbf{86.9 $\pm$ 1.2} && \textbf{99.5 $\pm$ 0.5} \\
    & \texttt{Blond\_Hair}        && 86.2 $\pm$ 1.8 && 90.8 $\pm$ 3.5 &&& 86.2 $\pm$ 1.6 && 89.7 $\pm$ 3.0 &&& \textbf{86.8 $\pm$ 1.1} && \textbf{98.8 $\pm$ 0.3} \\
    & \texttt{Heavy\_Makeup}      && 86.3 $\pm$ 1.8 && 92.8 $\pm$ 1.4 &&& 86.0 $\pm$ 1.6 && 94.4 $\pm$ 1.4 &&& \textbf{86.7 $\pm$ 1.2} && \textbf{99.4 $\pm$ 0.3} \\
    & \texttt{P+Y}         && 86.3 $\pm$ 1.8 && 88.0 $\pm$ 3.9 &&& 86.2 $\pm$ 1.9 && 91.5 $\pm$ 4.1 &&& \textbf{86.7 $\pm$ 1.1} && \textbf{98.8 $\pm$ 0.9} \\
    & \texttt{P+Y+B}   && 86.3 $\pm$ 1.8 && 85.6 $\pm$ 4.3 &&& 86.5 $\pm$ 1.5 && 88.7 $\pm$ 5.4 &&& \textbf{86.7 $\pm$ 1.3} && \textbf{98.4 $\pm$ 0.9} \\
    \midrule
    \multirow{3}{*}{\cite{li2021discover}} & $\attrvec_0$            && 86.2 $\pm$ 1.8 && 92.3 $\pm$ 2.1 &&& 86.3 $\pm$ 1.6 && 94.8 $\pm$ 3.7 &&& \textbf{86.9 $\pm$ 1.4} && \textbf{99.3 $\pm$ 0.9} \\
    & $\attrvec_0$+$\attrvec_1$                   && 86.3 $\pm$ 1.8	&& 90.7 $\pm$ 2.7 &&& 86.4 $\pm$ 1.5 && 93.4 $\pm$ 1.2 &&& \textbf{86.9 $\pm$ 1.1} && \textbf{98.3 $\pm$ 1.3} \\
    & $\attrvec_0$+$\attrvec_1$+$\attrvec_2$    && 86.3 $\pm$ 1.8 && 90.1 $\pm$ 2.8 &&& 86.3 $\pm$ 1.7 && 92.4 $\pm$ 1.6 &&& \textbf{86.8 $\pm$ 1.0} && \textbf{98.5 $\pm$ 0.6} \\
    \bottomrule
\end{tabular}
}
\end{table*}

\section{Certification with Ground Truth Data}
\label{app:3dshapes}

An essential part of the evaluation is demonstrating that the fairness
certificates obtained using the generative model can transfer to ground truth
data.
However, CelebA does not contain images of the same individual with different
attributes, \eg the same individual with different skin colors. Thus,
we experiment with the 3D
Shapes dataset (Apache-2.0 license)~\cite{burgess2018shapes}, which provides images
of 3D shapes that are procedurally generated from 6 independent latent
factors: \texttt{floor hue}, \texttt{wall hue}, \texttt{object hue},
\texttt{scale}, \texttt{shape}, and \texttt{orientation}.
Therefore, we can obtain ground truth
images of the same object with varying latent factors.
The 3D Shapes dataset is typically used to investigate disentanglement
properties of unsupervised learning methods, \eg in the context of
fairness~\cite{locatello2019fairness}.

The goal is to show that the similarity set computed by Glow captures
a given latent factor (as in~\cref{fig:shapes_appendix}) and that certification with
respect to this set will result in certification of the ground truth.
To that end, we experiment with \texttt{orientation} as the
continuous sensitive attribute. It has $v = 15$ possible
values, the most among all latent factors, providing for the most rigorous evaluation.
The target attribute is set to \texttt{object hue}, which has $10$ different classes.

We filter the original training set to
create a biased one, correlating \texttt{orienta-} \texttt{tion} and \texttt{object hue}.
We only keep those samples in the training set for which: (i) $\texttt{hue} \leq 5$ and
$\texttt{orient} \leq 7$, or (ii) $\texttt{hue} \geq 6$ and
$\texttt{orient} \geq 9$. We extend the attribute vector
computation from~\cref{sec:similarity}~\cite{kingma2018glow}
(performed on the original, unfiltered training set)
to non-binary attributes, defining $\attrvec_{ij} =
\generativepoint_{G, i} - \generativepoint_{G, j}$, where $1 \leq i, j \leq v$ are
sensitive attribute values. Based on the construction of the biased training set,
we let the similarity set $\simset$ to be defined by all attribute vectors
$\{\attrvec_{ij}\}$ for which $i < 8 < j$
($7 \cdot 7 = 49$ vectors) and set $\epsilon = 1$. We train Naive
($\lambda_1 = 1$; $\lambda_2 = \lambda_3 = 0$; $\RSsigma = 5$) and LASSI
($\lambda_1 = 1$; $\lambda_2 = 0.1$; $\lambda_3 = 0$;
$\RSsigma = 1$; $s = 100$) models and report results on 300 samples from the
test set. When running LASSI on 3D Shapes, we sample more points
($s = 100$) compared to the other datasets in order to accommodate for the more
complex similarity set, defined by many more attribute vectors.

In the evaluation, apart from reporting the accuracy and the certified fairness
(CertFair) on the (unbiased) test subset, for each sample we also obtain the $v$
similar ground truth data points, \ie the same shape at $v$ different orientations,
while fixing all other factors.
The empirical unfairness (EmpUnfair) in this case is the percentage of test samples
for which the downstream classifier does not classify all $v$ ground truth
individually similar images the same.
Moreover, if any of the $v$ similar data points is
\emph{certified}, we check whether \emph{all} $v$ similar ground truth data points obtain
the same classification, indicating ground truth fairness.

\cref{tbl:shapes} shows that
LASSI substantially increases the accuracy and the certified individual fairness
(\wrt the similarity set computed using Glow),
while being nearly 100\% empirically fair on the ground truth images.
That is, in 0.3\%
of the test samples there were different classification outcomes among their $v$
similar (ground-truth) samples. Crucially, in all of these cases, our method
did not certify individual fairness for \emph{any} of
the $v$ similar data points, showing that the certificates transfer to
the ground truth.

\begin{table}
	\caption{
		Evaluation on 3D Shapes for the task \texttt{object hue}.
		The certification rate (CertFair) and the percentage of ground
		truth empirically unfair data points (EmpFair) sum up below 100\%.
	}
	\begin{center}
		\resizebox{\linewidth}{!}{\begin{tabular}{@{}lcccccccccccccc@{}}
    \toprule
    Method: &&& \multicolumn{5}{c}{Naive} &&& \multicolumn{5}{c}{LASSI (ours)} \\
    \cmidrule{4-8}
    \cmidrule{11-15}
    Sens. attrib. &&& Acc && CertFair && EmpUnfair ($\downarrow$) &&& Acc && CertFair && EmpUnfair ($\downarrow$) \\
    \midrule
    \texttt{orientation} &&& 32.0 && 0 && 69.3 &&& \textbf{100} && \textbf{81.3} && \textbf{0.3} \\
    \bottomrule
\end{tabular}
}
	\end{center}
	\label{tbl:shapes}
\end{table}

\section{More Examples of Similar Individuals}
\label{app:glow-reconstructions}

Here, we provide further samples from the similarity sets $\simseti$ (defined with
$\attrvec = \generativepoint_{G, pos} - \generativepoint_{G, neg}$), as
reconstructed by Glow, for various inputs $\inputpoint$ randomly drawn from our
evaluation subsets. A summary of all configurations
is listed in~\cref{tbl:list-reconstructions}. The images in the middle of the CelebA
and FairFace reconstructions correspond to the original inputs. The perturbations
range uniformly between $[-\frac{\epsilon}{\sqrt{n}}, \frac{\epsilon}{\sqrt{n}}]$,
where $n$ is the number of sensitive attributes. For $n > 1$, all
attribute vectors are multiplied by the same $t$ before adding them to the latent
representation of the original inputs. $\epsilon=1$ for CelebA and 3D Shapes
and $\epsilon=0.5$ for FairFace.

\begin{table}[tbh!]
    \caption{
		Example image reconstructions from the similarity sets in this work.
	}
	\begin{center}
		\resizebox{0.55\linewidth}{!}{\begin{tabular}{@{}lp{0.1cm}lp{0.1cm}l@{}}
    \toprule
    Dataset && Sensitive attribute(s) && Figure \\
    \midrule
    \multirow{5}{*}{CelebA} && \texttt{Pale\_Skin} && \cref{fig:appendix-pale-skin-0} \\
    && \texttt{Young} && \cref{fig:appendix-young-0} \\
    && \texttt{Blond\_Hair} && \cref{fig:appendix-blond-hair-0} \\
    && \texttt{Heavy\_Makeup} && \cref{fig:appendix-heavy-makeup-0} \\
    && \texttt{Pale} \texttt{+} \texttt{Young} && \cref{fig:appendix-pale-skin-young-0} \\
    && \texttt{Pale} \texttt{+} \texttt{Young} \texttt{+} \texttt{Blond} && \cref{fig:appendix-pale-young-blond-0} \\
    \midrule
    FairFace && \texttt{Race=Black} && \cref{fig:appendix-fairface-0} \\
    \midrule
    3D Shapes && \texttt{orientation} && \cref{fig:shapes_appendix} \\
    \bottomrule
\end{tabular}
}
	\end{center}
	\label{tbl:list-reconstructions}
\end{table}

\begin{figure*}[hbt!]
	\centering
	\includegraphics[width=\linewidth]{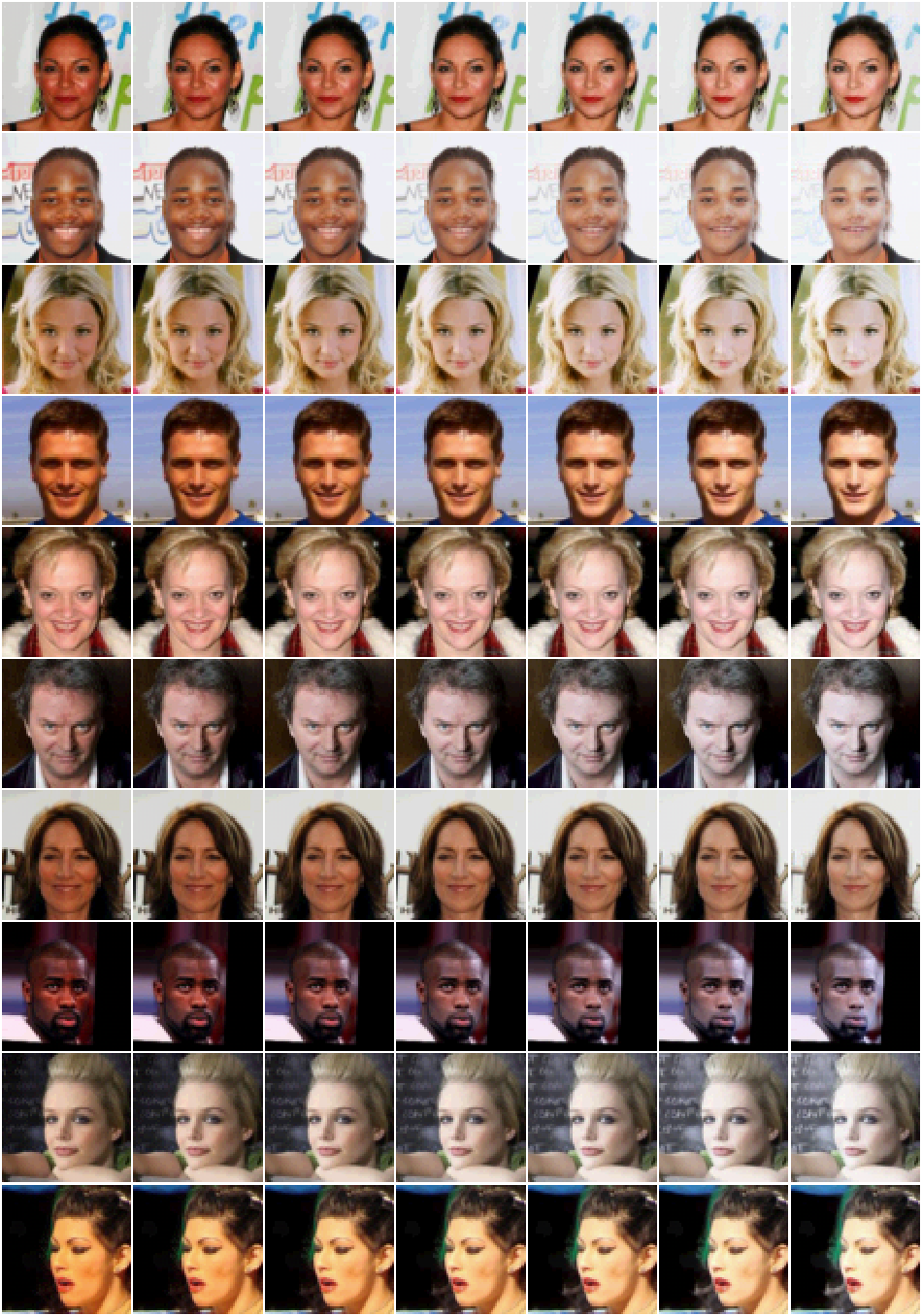}
	\caption{
		Similar individuals from $\simseti$, for $\inputpoint$ in the CelebA dataset,
		obtained by varying the sensitive attribute \texttt{Pale\_Skin}.
	}
	\label{fig:appendix-pale-skin-0}
\end{figure*}

\begin{figure*}[hbt!]
	\centering
	\includegraphics[width=\linewidth]{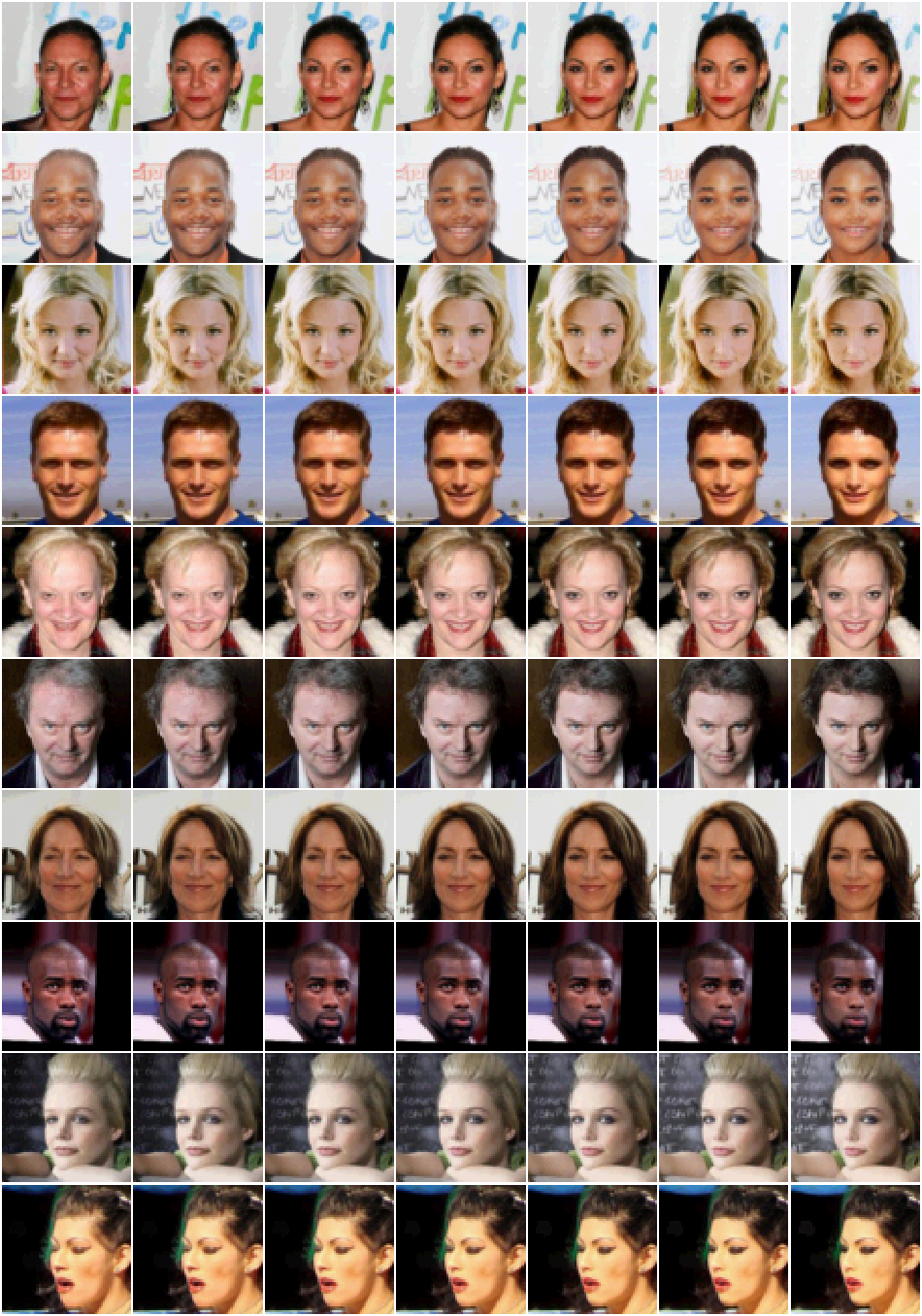}
	\caption{
		Similar individuals from $\simseti$, for $\inputpoint$ in the CelebA dataset,
		obtained by varying the sensitive attribute \texttt{Young}.
	}
	\label{fig:appendix-young-0}
\end{figure*}

\begin{figure*}[hbt!]
	\centering
	\includegraphics[width=\linewidth]{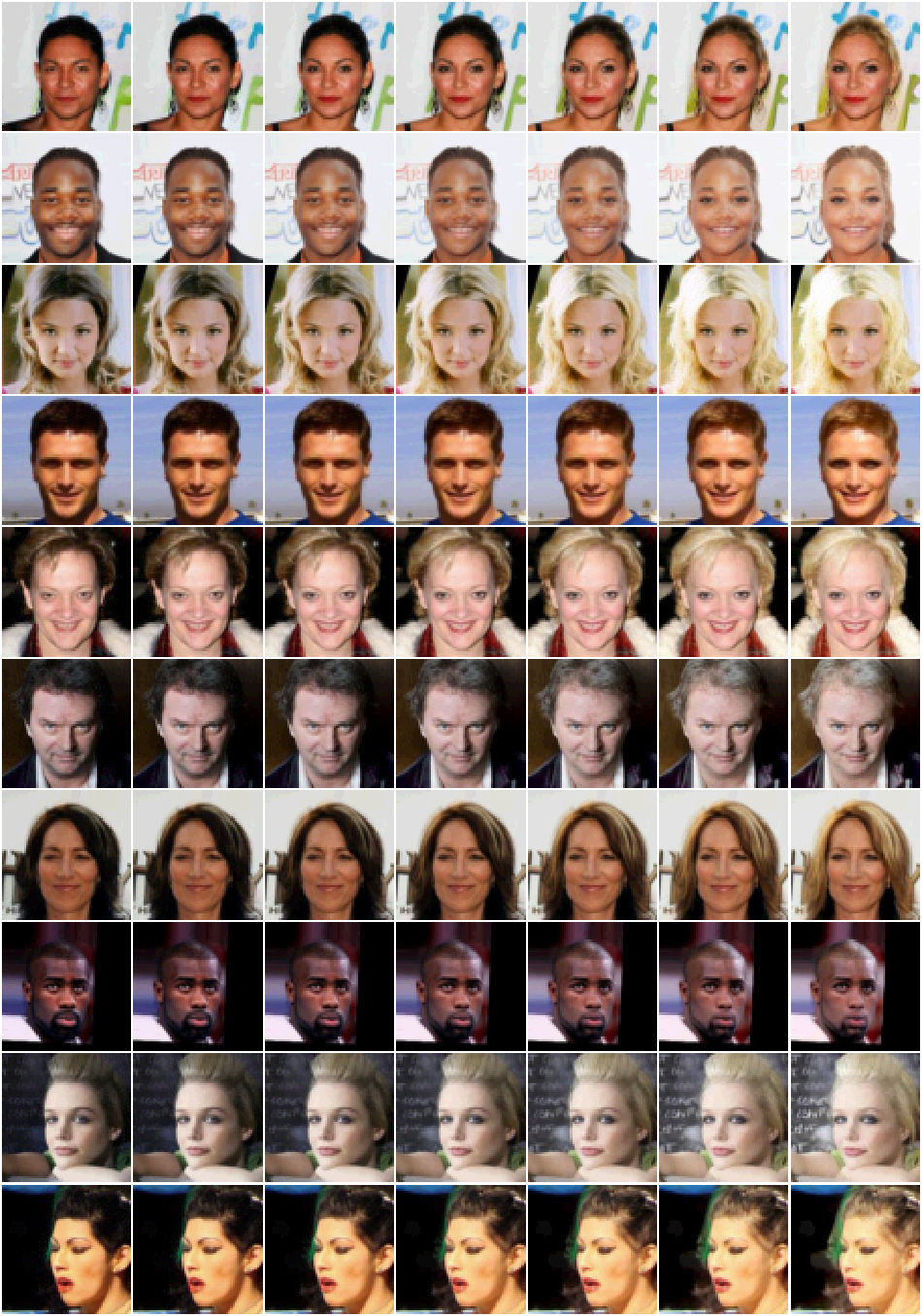}
	\caption{
		Similar individuals from $\simseti$, for $\inputpoint$ in the CelebA dataset,
		obtained by varying the sensitive attribute \texttt{Blond\_Hair}.
	}
	\label{fig:appendix-blond-hair-0}
\end{figure*}

\begin{figure*}[hbt!]
	\centering
	\includegraphics[width=\linewidth]{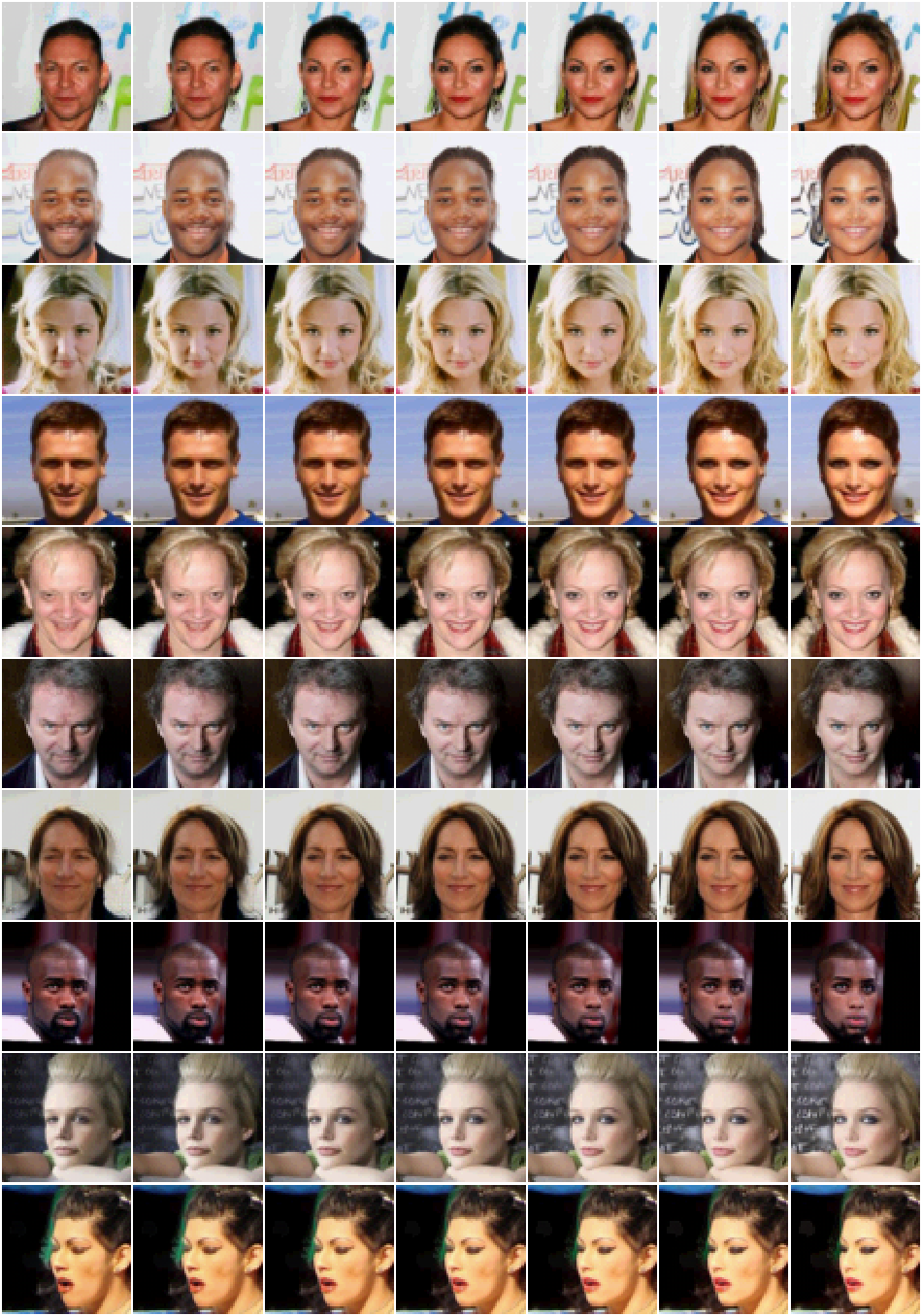}
	\caption{
		Similar individuals from $\simseti$, for $\inputpoint$ in the CelebA dataset,
		obtained by varying the sensitive attribute \texttt{Heavy\_Makeup}.
	}
	\label{fig:appendix-heavy-makeup-0}
\end{figure*}

\begin{figure*}[hbt!]
	\centering
	\includegraphics[width=\linewidth]{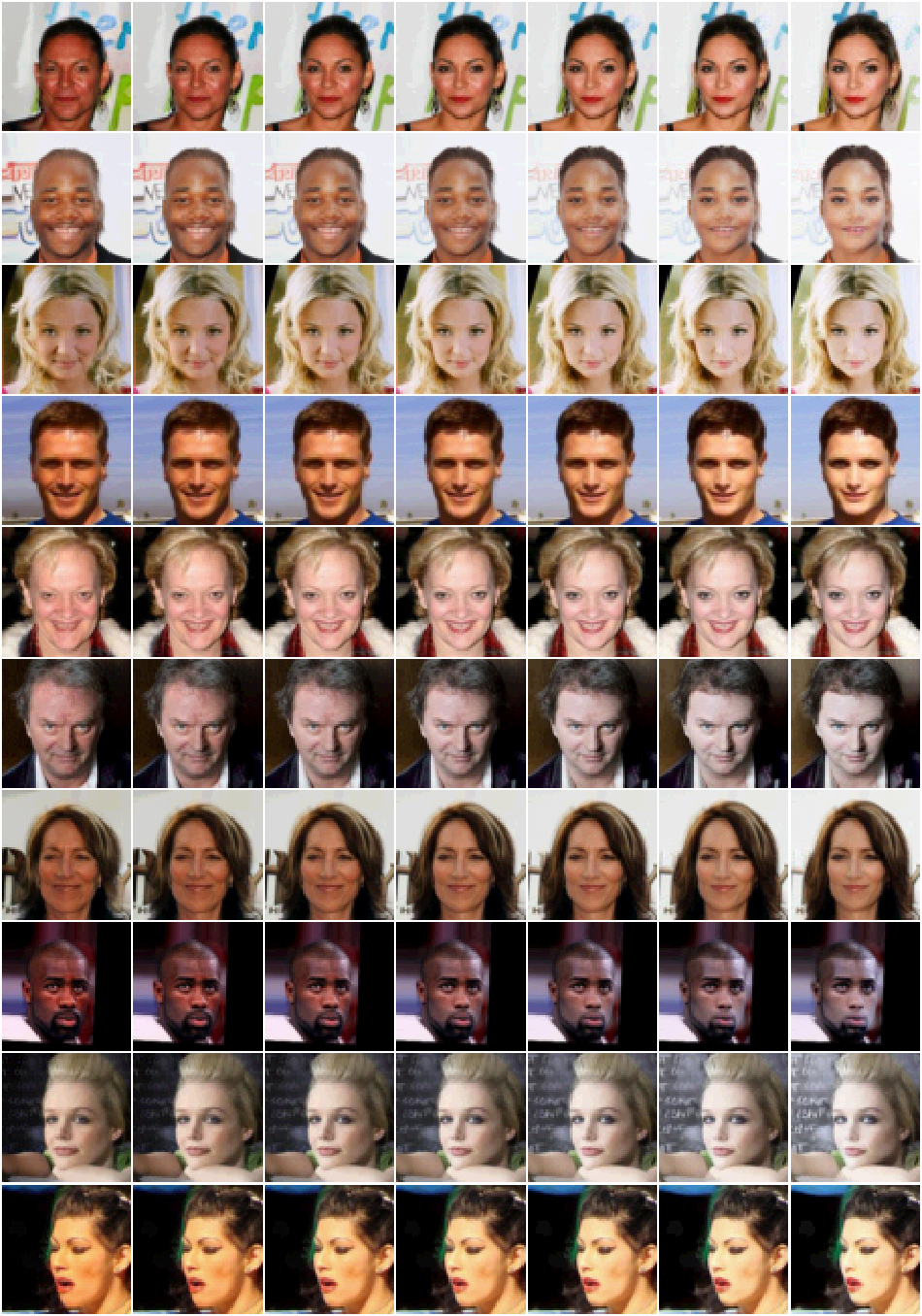}
	\caption{
		Similar individuals from $\simseti$ obtained by simultaneously varying the
		sensitive attributes \texttt{Pale\_Skin} \texttt{+} \texttt{Young}.
	}
	\label{fig:appendix-pale-skin-young-0}
\end{figure*}

\begin{figure*}[hbt!]
	\centering
	\includegraphics[width=\linewidth]{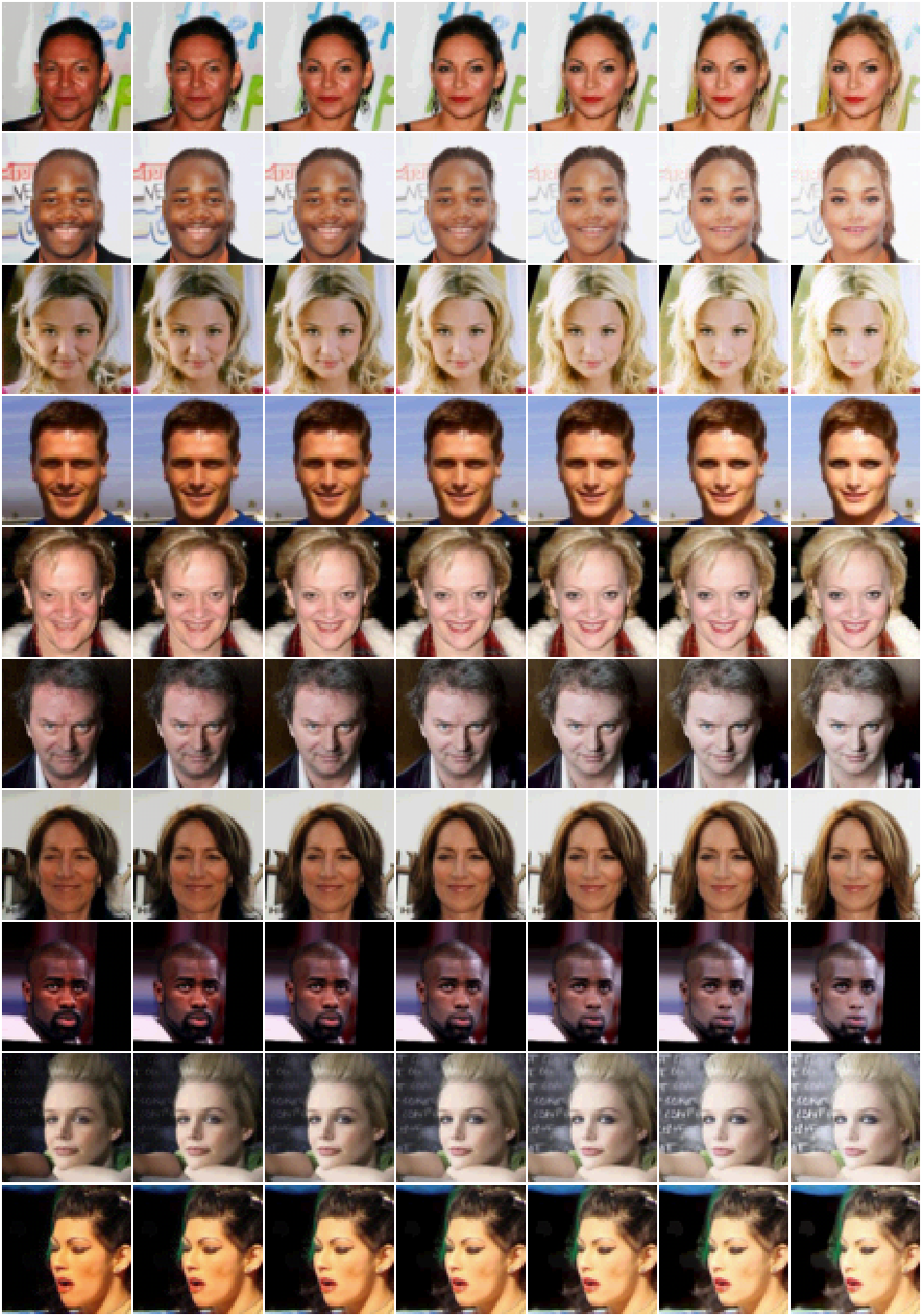}
	\caption{
		Similar individuals from $\simseti$ obtained by simultaneously varying the
		sensitive attributes \texttt{Pale\_Skin} \texttt{+} \texttt{Young} \texttt{+}
		\texttt{Blond}.
	}
	\label{fig:appendix-pale-young-blond-0}
\end{figure*}

\begin{figure*}[hbt!]
	\centering
	\includegraphics[width=\linewidth]{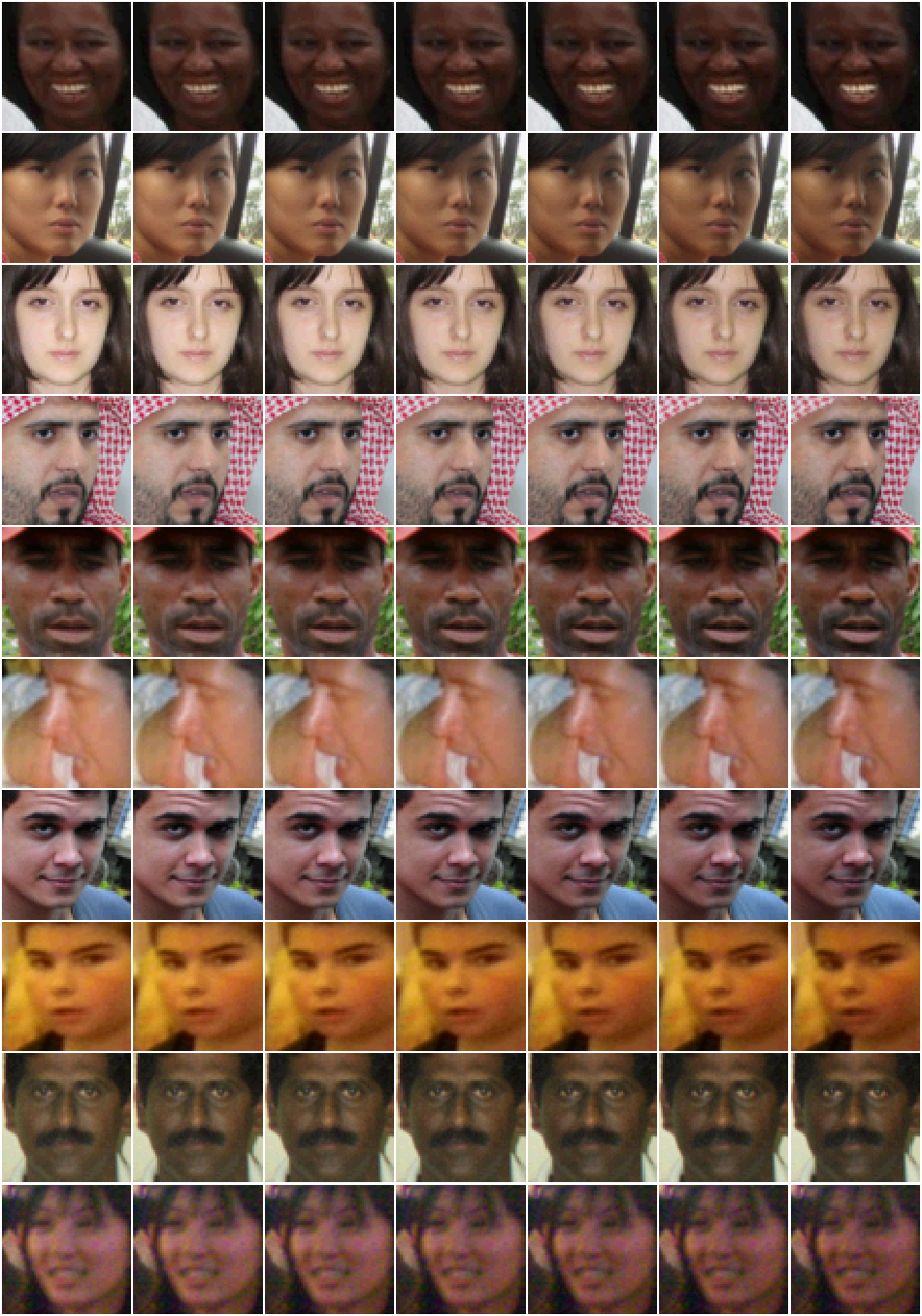}
	\caption{
		Similar individuals from $\simseti$, for $\inputpoint$ in FairFace
		and $\epsilon=0.5$, obtained by varying the sensitive attribute
		\texttt{Race=Black}.
	}
	\label{fig:appendix-fairface-0}
\end{figure*}

\begin{figure}[hbt!]
	\centering
	\subfloat{%
		\includegraphics[width=\linewidth]{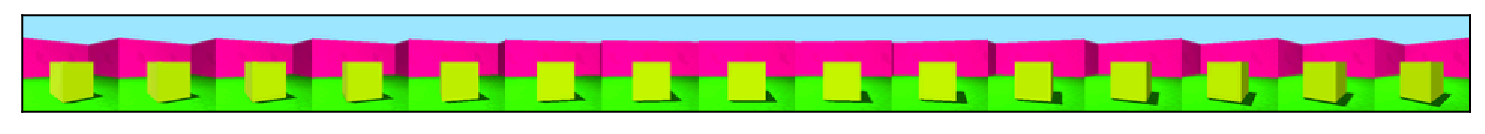}%
	}\hfil
	\vspace{-0.2cm}
	\subfloat{%
		\includegraphics[width=\linewidth]{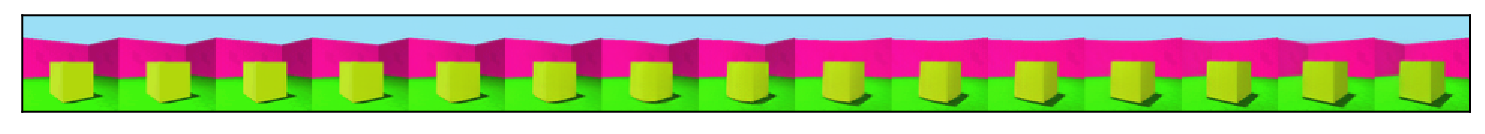}%
	}\hfil

	\subfloat{%
		\includegraphics[width=\linewidth]{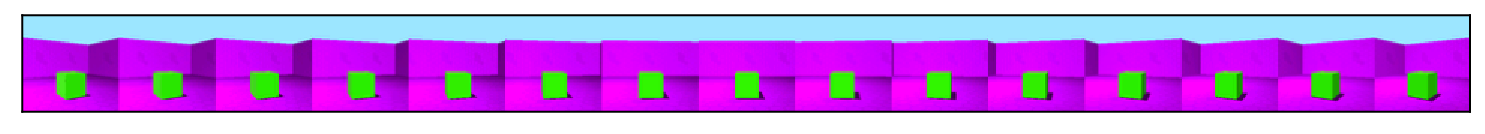}%
	}\hfil
	\vspace{-0.2cm}
	\subfloat{%
		\includegraphics[width=\linewidth]{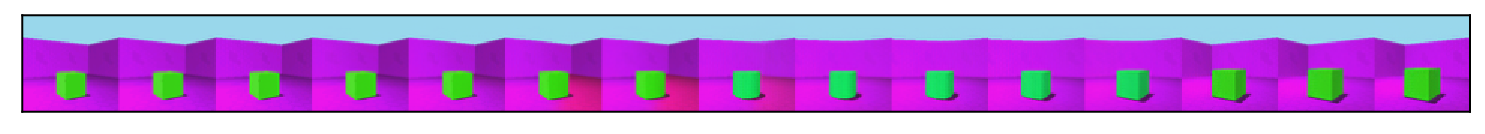}%
	}\hfil

	\subfloat{%
		\includegraphics[width=\linewidth]{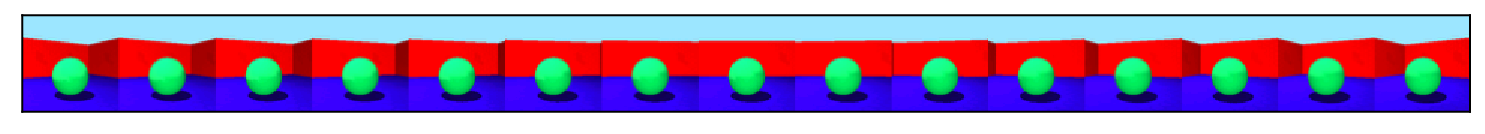}%
	}\hfil
	\vspace{-0.2cm}
	\subfloat{%
		\includegraphics[width=\linewidth]{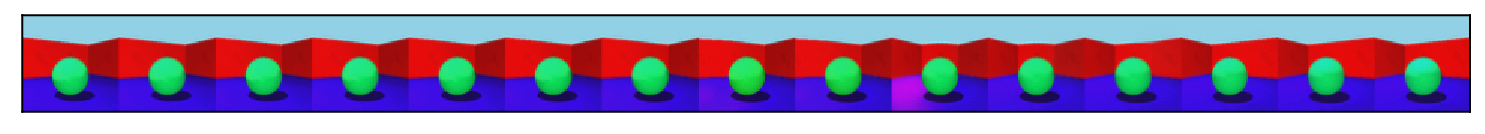}%
	}\hfil

	\subfloat{%
		\includegraphics[width=\linewidth]{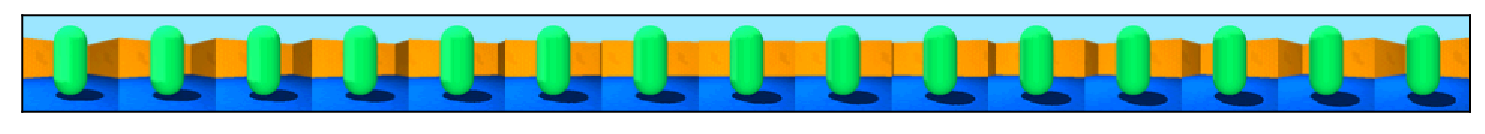}%
	}\hfil
	\vspace{-0.2cm}
	\subfloat{%
		\includegraphics[width=\linewidth]{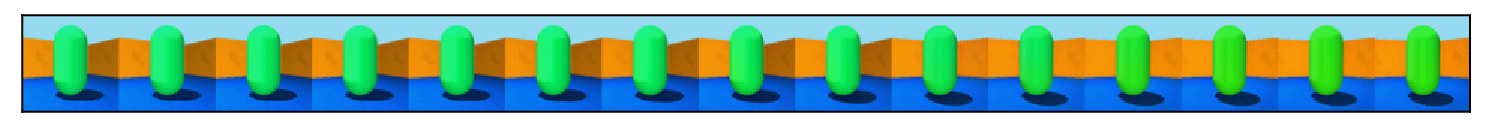}%
	}\hfil

	\subfloat{%
		\includegraphics[width=\linewidth]{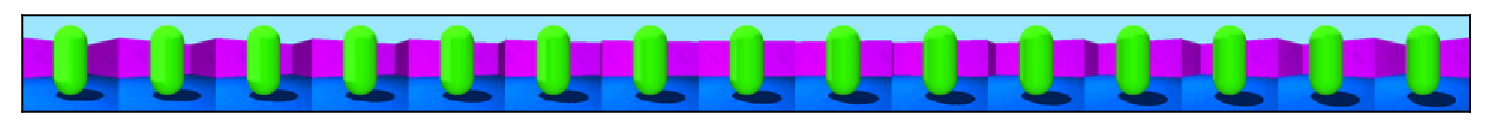}%
	}\hfil
	\vspace{-0.2cm}
	\subfloat{%
		\includegraphics[width=\linewidth]{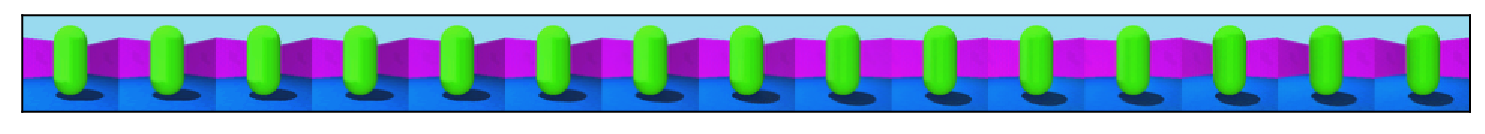}%
	}\hfil

	\subfloat{%
		\includegraphics[width=\linewidth]{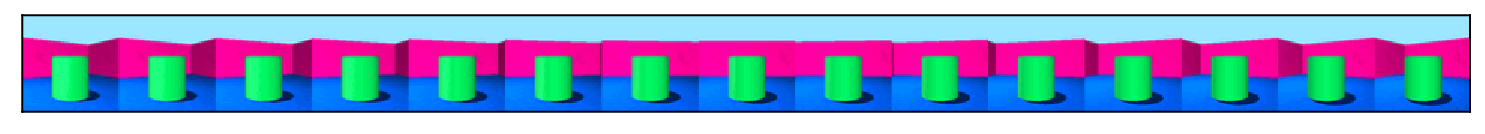}%
	}\hfil
	\vspace{-0.2cm}
	\subfloat{%
		\includegraphics[width=\linewidth]{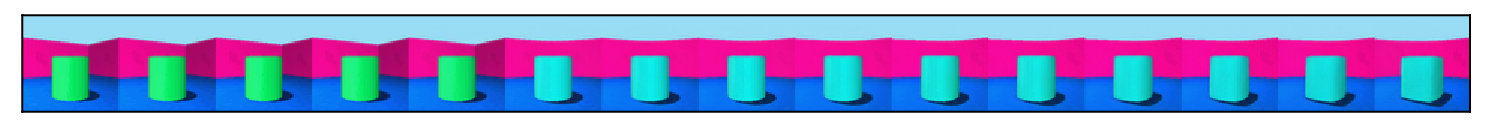}%
	}

	\caption{
		Sampled shapes at 15 different ground truth orientations. The original (above)
		and the corresponding reconstructions (below) obtained from interpolating along
		one of the attribute vectors, $\attrvec_{1,15}$
		(see~\cref{app:3dshapes} for details), grouped together.
	}
	\label{fig:shapes_appendix}
\end{figure}

}{}

\message{^^JLASTPAGE \thepage^^J}

\end{document}